\documentclass{article}

\usepackage[round, authoryear]{natbib}
    \PassOptionsToPackage{numbers, compress}{natbib}

    \usepackage[final]{neurips_2020}

\usepackage[utf8]{inputenc} 
\usepackage[T1]{fontenc}    
\usepackage{hyperref}       
\usepackage{url}            
\usepackage{booktabs}       
\usepackage{amsfonts}       
\usepackage{nicefrac}       
\usepackage{microtype}      

\usepackage{amsmath}
\usepackage{amssymb}
\usepackage{amsthm}

\usepackage{enumerate}
\usepackage[usenames, dvipsnames]{color}

\usepackage{graphicx}
\usepackage{bm}
\usepackage[toc,page]{appendix}

\usepackage{subfigure}
\usepackage{adjustbox}

\graphicspath{{plots/}}

\usepackage{algorithm}
\usepackage[noend]{algpseudocode}

\algnewcommand\algorithmicinput{\textbf{Initialize:}}
\algnewcommand\INPUT{\item[\algorithmicinput]}

\usepackage{thm-restate}

\usepackage{wrapfig}
\usepackage{changepage}
\usepackage{float}
\usepackage{bm}

\usepackage{multirow}

\usepackage{relsize}

\DeclareMathOperator*{\E}{\mathbb{E}}
\DeclareMathOperator*{\proj}{\text{proj}}

\DeclareMathOperator*{\argmax}{argmax}

\newcommand{\cL}{\mathcal{L}}

\newcommand{\cS}{\mathcal{S}}

\newcommand{\cA}{\mathcal{A}}

\newcommand{\cC}{\mathcal{C}}

\newcommand{\1}{\boldsymbol{1}}

\newcommand{\R}{\mathbb{R}}

\newcommand{\pol}{\pi_{\theta}}

\newcommand{\polk}{\pi_{\theta_k}}
\newcommand{\polkup}{\pi_{\theta_{k+1}}}
\newcommand{\grad}{\nabla_{\theta}}
\newcommand{\gradnu}{\frac{\partial}{\partial\nu}}

\newcommand{\advk}{A^{\polk}}

\newcommand{\dpolk}{d^{\pi_{\theta_k}}}

\newcommand{\KL}[2]{D_\text{KL}\left(#1\middle\| #2\right)}
\newcommand{\avKL}[2]{\bar{D}_\text{KL}(#1\parallel #2)}

\newcommand{\clip}{\text{clip}}

\title{First Order Constrained Optimization in Policy Space}

\author{%
  Yiming Zhang \\
  New York University \\
  \texttt{yiming.zhang@cs.nyu.edu} \\
  \And
  Quan Vuong \\
  UC San Diego \\
  \texttt{qvuong@ucsd.edu} \\
  \And
  Keith W. Ross \\
  New York University Shanghai \\
  New York University \\
  \texttt{keithwross@nyu.edu} \\
}

\begin{document}

\maketitle

\begin{abstract}
In reinforcement learning, an agent attempts to learn high-performing behaviors through interacting with the environment, such behaviors are often quantified in the form of a reward function. However some aspects of behavior\textemdash such as ones which are deemed unsafe and to be avoided\textemdash are best captured through constraints. We propose a novel approach called First Order Constrained Optimization in Policy Space (FOCOPS) which maximizes an agent's overall reward while ensuring the agent satisfies a set of cost constraints. Using data generated from the current policy, FOCOPS first finds the optimal update policy by solving a constrained optimization problem in the nonparameterized policy space. FOCOPS then projects the update policy back into the parametric policy space. Our approach has an approximate upper bound for worst-case constraint violation throughout training and is first-order in nature therefore simple to implement. We provide empirical evidence that our simple approach achieves better performance on a set of constrained robotics locomotive tasks.
\end{abstract}

\section{Introduction}

In recent years, Deep Reinforcement Learning (DRL) saw major breakthroughs in several challenging high-dimensional tasks such as Atari games \citep{mnih2013playing,mnih2016asynchronous,van2016deep,schaul2015prioritized,wang2017sample}, playing go \citep{silver2016mastering,silver2018mastering}, and robotics \citep{peters2008reinforcement,schulman2015trust,schulman2017proximal,wu2017scalable,haarnoja2018soft}. However most modern DRL algorithms allow the agent to freely explore the 
environment to obtain desirable behavior, provided that it leads to performance improvement. No regard is given to whether the agent's behavior may lead to negative or harmful consequences. Consider for instance the task of controlling a robot, certain maneuvers may damage the robot, or worse harm people around it. RL safety \citep{amodei2016concrete} is a pressing topic in modern reinforcement learning research and imperative to applying reinforcement learning to real-world settings.

Constrained Markov Decision Processes (CMDP) \citep{kallenberg1983linear,ross1985constrained,beutler1985optimal,ross1989markov, altman1999constrained} provide a principled mathematical framework for dealing with such problems as it allows us to naturally incorporate safety criteria in the form of constraints. In low-dimensional finite settings, an optimal policy for CMDPs with known dynamics can be found by linear programming \citep{kallenberg1983linear} or Lagrange relaxation \citep{ross1985constrained, beutler1985optimal}.

While we can solve problems with small state and action spaces via linear programming and value iteration, function approximation is required in order to generalize over large state spaces. Based on recent advances in local policy search methods \citep{kakade2002approximately,peters2008reinforcement, schulman2015trust}, \cite{achiam2017constrained} proposed the Constrained Policy Optimization (CPO) algorithm. However policy updates for the CPO algorithm involve solving an optimization problem through Taylor approximations and inverting a high-dimensional Fisher information matrix. These approximations often result in infeasible updates which would require additional recovery steps, this could sometimes cause updates to be backtracked leading to a waste of samples.

In this paper, we propose the First Order Constrained Optimization in Policy Space (FOCOPS) algorithm. FOCOPS attempts to answer the following question: given some current policy, what is the best constraint-satisfying policy update? FOCOPS provides a solution to this question in the form of a two-step approach. First, we will show that the best policy update has a near-closed form solution when attempting to solve for the optimal policy in the nonparametric policy space rather than the parameter space. However in most cases, this optimal policy is impossible to evaluate. Hence we project this policy back into the parametric policy space. This can be achieved by drawing samples from the current policy and evaluating a loss function between the parameterized policy and the optimal policy we found in the nonparametric policy space. 
Theoretically, FOCOPS has an approximate upper bound for worst-case constraint violation throughout training. Practically, FOCOPS is extremely simple to implement since it only utilizes first order approximations. We further test our algorithm on a series of challenging high-dimensional continuous control tasks and found that FOCOPS achieves better performance while maintaining approximate constraint satisfaction compared to current state of the art approaches, in particular second-order approaches such as CPO. 

\section{Preliminaries}
\subsection{Constrained Markov Decision Process}
Consider a Markov Decision Process (MDP) \citep{sutton2018reinforcement} denoted by the tuple $(\cS, \cA, R, P, \mu)$ where $\cS$ is the state space, $\cA$ is the action space, $P:\cS\times\cA\times \cS\to[0,1]$ is the transition kernel, $R:\cS\times\cA\to\R$ is the reward function, $\mu:\cS\to [0,1]$ is the initial state distribution. Let $\pi=\{\pi(a|s):s\in\cS, a\in\cA\}$ denote a policy, and $\Pi$ denote the set of all stationary policies. We aim to find a stationary policy that maximizes the expected discount return $J(\pi) := \E_{\tau\sim\pi}\left[\sum_{t=0}^{\infty}\gamma^t R(s_t,a_t) \right]$.
Here $\tau=(s_0,a_0,\dots,)$ is a sample trajectory and $\gamma\in (0,1)$ is the discount factor. We use $\tau\sim\pi$ to indicate that the trajectory distribution depends on $\pi$ where $s_0\sim\mu$, $a_t\sim\pi(\cdot|s_t)$, and $s_{t+1}\sim P(\cdot|s_t,a_t)$. The value function is expressed as $V^{\pi}(s):=\E_{\tau\sim\pi}\left[\sum_{t=0}^{\infty}\gamma^t R(s_t,a_t)\bigg| s_0=s\right]$ and action-value function as $Q^{\pi}(s,a):= \E_{\tau\sim\pi}\left[\sum_{t=0}^{\infty}\gamma^t R(s_t,a_t)\bigg| s_0=s, a_0=a\right]$.
The advantage function is defined as $A^{\pi}(s,a) := Q^{\pi}(s,a)-V^{\pi}(s)$.
Finally, we define the discounted future state visitation distribution as $d^{\pi}(s) := (1-\gamma)\sum_{t=0}^{\infty}\gamma^t P(s_t=s|\pi)$.

A Constrained Markov Decision Process (CMDP) \citep{kallenberg1983linear,ross1985constrained,altman1999constrained} is an MDP with an additional set of constraints $\cC$ which restricts the set of allowable policies. The set $\cC$ consists of a set of cost functions $C_i:\cS\times\cA\to\R$, $i=1,\dots,m$. Define the $C_i$\textit{-return} as $J_{C_i}(\pi) := \E_{\tau\sim\pi}\left[\sum_{t=0}^{\infty}\gamma^t C_i(s,a)\right]$.
The set of feasible policies is then $\Pi_{\cC} := \{\pi\in\Pi:J_{C_i}(\pi)\leq b_i,i=1,\dots,m\}$. The reinforcement learning problem w.r.t. a CMDP is to find a policy such that $\pi^* = \argmax_{\pi\in\Pi_{\cC}}J(\pi)$.

Analogous to the standard $V^{\pi}$, $Q^{\pi}$, and $A^{\pi}$ for return, we define the cost value function, cost action-value function, and cost advantage function as $V_{C_i}^{\pi}$, $Q_{C_i}^{\pi}$, and $A_{C_i}^{\pi}$ where we replace the reward $R$ with $C_i$. Without loss of generality, we will restrict our discussion to the case of one constraint with a cost function $C$. However we will briefly discuss in later sections how our methodology can be naturally extended to the multiple constraint case.

\subsection{Solving CMDPs via Local Policy Search}

Typically, we update policies by drawing samples from the environment, hence we usually consider a set of parameterized policies (for example, neural networks with a fixed architecture) $\Pi_\theta=\{\pol:\theta\in\Theta\}\subset\Pi$ from which we can easily evaluate and sample from. Conversely throughout this paper, we will also refer to $\Pi$ as the \textit{nonparameterized policy space}. 

Suppose we have some policy update procedure and we wish to update the policy at the $k$th iteration $\polk$ to obtain $\polkup$. Updating $\polk$ within some local region (i.e. $D(\pol, \polk)<\delta$ for some divergence measure $D$) can often lead to more stable behavior and better sample efficiency \citep{peters2008reinforcement,kakade2002approximately,pirotta2013safe}. In particular, theoretical guarantees for policy improvement can be obtained when $D$ is chosen to be $\KL{\pol}{\polk}$ \citep{schulman2015trust,achiam2017constrained}. 

However solving CMDPs directly within the context of local policy search can be challenging and sample inefficient since after each policy update, additional samples need to be collected from the new policy in order to evaluate whether the $C$ constraints are satisfied. \cite{achiam2017constrained} proposed replacing the cost constraint with a surrogate cost function which evaluates the constraint $J_{C}(\pol)$ using samples collected from the current policy $\polk$. This surrogate function is shown to be a good approximation to $J_{C}(\pol)$ when $\pol$ and $\polk$ are close w.r.t. the KL divergence.
Based on this idea, the CPO algorithm \citep{achiam2017constrained} performs policy updates as follows: given some policy $\polk$, the new policy $\polkup$ is obtained by solving the optimization problem
\begin{align}
\underset{\pol\in\Pi_{\theta}}{\text{maximize}} \quad
& \E_{\substack{s\sim\dpolk\\ a\sim\pol}}[A^{\polk}(s,a)] \label{eq:CPO_obj}  \\
\text{subject to} \quad
& J_{C}(\polk)+\frac{1}{1-\gamma}\E_{\substack{s\sim\dpolk \\ a\sim\pol}}\left[A^{\polk}_{C}(s,a)\right]\leq b \label{eq:CPO_cost_const} \\
& \avKL{\pol}{\polk}\leq \delta. \label{eq:CPO_kl_const}  
\end{align}
where $\avKL{\pol}{\polk} := \E_{s\sim\dpolk}[\KL{\pol}{\polk}[s]]$. We will henceforth refer to constraint \eqref{eq:CPO_cost_const} as the \textit{cost constraint} and \eqref{eq:CPO_kl_const} as the \textit{trust region constraint}. For policy classes with a high-dimensional parameter space such as deep neural networks, it is often infeasible to solve Problem (\ref{eq:CPO_obj}-\ref{eq:CPO_kl_const}) directly in terms of $\theta$. \cite{achiam2017constrained} solves Problem (\ref{eq:CPO_obj}-\ref{eq:CPO_kl_const}) by first applying first and second order Taylor approximation to the objective and constraints, the resulting optimization problem is convex and can be solved using standard convex optimization techniques. 

However such an approach introduces several sources of error, namely (i) Sampling error resulting from taking sample trajectories from the current policy (ii) Approximation errors resulting from Taylor approximations (iii) Solving the resulting optimization problem post-Taylor approximation involves taking the inverse of a Fisher information matrix, whose size is equal to the number of parameters in the policy network. Inverting such a large matrix is computationally expensive to attempt directly hence the CPO algorithm uses the conjugate gradient method \citep{strang2007computational} to indirectly calculate the inverse. This results in further approximation errors. In practice the presence of these errors require the CPO algorithm to take additional steps during each update in the training process in order to recover from constraint violations, this is often difficult to achieve and may not always work well in practice. We will show in the next several sections that our approach is able to eliminate the last two sources of error and outperform CPO using a simple first-order method.

\subsection{Related Work}
In the tabular case, CMDPs have been extensively studied for different constraint criteria \citep{kallenberg1983linear,beutler1985optimal,beutler1986time,ross1989randomized,ross1989markov,ross1991multichain,altman1999constrained}.

In high-dimensional settings, \cite{chow2017risk} proposed a primal-dual method which is shown to converge to policies satisfying cost constraints. \cite{tessler2018reward} introduced a penalized reward formulation and used a multi-timescaled approach for training an actor-critic style algorithm which guarantees convergence to a fixed point. However multi-timescaled approaches impose stringent requirements on the learning rates which can be difficult to tune in practice. We note that neither of these methods are able to guarantee cost constraint satisfaction during training.

Several recent work leveraged advances in control theory to solve the CMDP problem. \cite{chow2018lyapunov,chow2019lyapunov} presented a method for constructing Lyapunov function which guarantees constraint-satisfaction during training. \cite{stooke2020responsive} combined PID control with Lagrangian methods which dampens cost oscillations resulting in reduced constraint violations.

Recently \cite{yang2020projection} independently proposed the Projection-Based Constrained Policy Optimization (PCPO) algorithm which utilized a different two-step approach. PCPO first finds the policy with the maximum return by doing one TRPO \citep{schulman2015trust} update. It then projects this policy back into the feasible cost constraint set in terms of the minimum KL divergence. While PCPO also satisfies theoretical guarantees for cost constraint satisfaction, it uses second-order approximations in both steps. FOCOPS is first-order which results in a much simpler algorithm in practice. Furthermore, empirical results from PCPO does not consistently outperform CPO.

The idea of first solving within the nonparametric space and then projecting back into the parameter space has a long history in machine learning and has recently been adopted by the RL community. \cite{abdolmaleki2018maximum} took the “inference view” of policy search and attempts to find the desired policy via the EM algorithm, whereas FOCOPS is motivated by the “optimization view” by directly solving the cost-constrained trust region problem using a primal-dual approach then projecting the solution back into the parametric policy space. \cite{peters2010relative} and \cite{montgomery2016guided} similarly took an optimization view but are motivated by different optimization problems. \cite{vuong2018supervised} proposed a general framework exploring different trust-region constraints. However to the best of our knowledge, FOCOPS is the first algorithm to apply these ideas to cost-constrained RL.

\section{Constrained Optimization in Policy Space}

Instead of solving (\ref{eq:CPO_obj}-\ref{eq:CPO_kl_const}) directly, we use a two-step approach summarized below:
\begin{enumerate}
    \item Given policy $\polk$, find an \textit{optimal update policy} $\pi^*$ by solving the optimization problem from (\ref{eq:CPO_obj}-\ref{eq:CPO_kl_const}) in the nonparameterized policy space.
    \item Project the policy found in the previous step back into the parameterized policy space $\Pi_{\theta}$ by solving for the closest policy $\pol\in\Pi_{\theta}$ to $\pi^*$ in order to obtain $\pi_{\theta_{k+1}}$.  
\end{enumerate}

\subsection{Finding the Optimal Update Policy}

In the first step, we consider the optimization problem
\begin{align}
\underset{\pi\in\Pi}{\text{maximize}} \quad
& \E_{\substack{s\sim\dpolk\\ a\sim\pi}}[A^{\polk}(s,a)] \label{eq:FOCOPS_obj} \\
\text{subject to} \quad
& J_{C}(\polk)+\frac{1}{1-\gamma}\E_{\substack{s\sim\dpolk \\ a\sim\pi}}\left[A^{\polk}_{C}(s,a)\right]\leq b \label{eq:FOCOPS_cost_const} \\
& \avKL{\pi}{\polk}\leq \delta  \label{eq:FOCOPS_tr_const}
\end{align}
Note that this problem is almost identical to Problem (\ref{eq:CPO_obj}-\ref{eq:CPO_kl_const}) except  the parameter of interest is now the nonparameterized policy $\pi$ and not the policy parameter $\theta$. We can show that Problem (\ref{eq:FOCOPS_obj}-\ref{eq:FOCOPS_tr_const}) admits the following solution (see Appendix \ref{append:thm} of the supplementary material for proof):
\begin{restatable}{theorem}{FOCOPSsol}\label{thm:FOCOPS_opt_pol}
Let $\Tilde{b}=(1-\gamma)(b - \Tilde{J}_{C}(\polk))$. If $\polk$ is a feasible solution, the optimal policy for (\ref{eq:FOCOPS_obj}-\ref{eq:FOCOPS_tr_const}) takes the form
\begin{equation}\label{eq:FOCOPS_opt}
        \pi^{*}(a|s) = \frac{\polk(a|s)}{Z_{\lambda,\nu}(s)}\exp\left(\frac{1}{\lambda}\left(A^{\polk}(s,a)-\nu A^{\polk}_{C}(s,a) \right)\right)
\end{equation}
where $Z_{\lambda,\nu}(s)$ is the partition function which ensures \eqref{eq:FOCOPS_opt} is a valid probability distribution, $\lambda$ and $\nu$ are solutions to the optimization problem:
\begin{equation}\label{eq:dual}
    \min_{\lambda,\nu\geq 0}\;\lambda\delta+\nu\Tilde{b} + \lambda\E_{\substack{s\sim\dpolk\\ a\sim\pi^*}}[\log Z_{\lambda,\nu}(s)]
\end{equation}
\end{restatable}
The form of the optimal policy is intuitive, it gives high probability mass to areas of the state-action space with high return which is offset by a penalty term times the cost advantage. We will refer to the optimal solution to (\ref{eq:FOCOPS_obj}-\ref{eq:FOCOPS_tr_const}) as the \textit{optimal update policy}. We also note that it is possible to extend our results to accommodate for multiple constraints by introducing Lagrange multipliers $\nu_1,\dots,\nu_m\geq 0$, one for each cost constraint and applying a similar duality argument. 

Another desirable property of the optimal update policy $\pi^*$ is that for any feasible policy $\polk$, it satisfies the following bound for worst-case guarantee for cost constraint satisfaction from \cite{achiam2017constrained}:
\begin{equation}
    J_{C}(\pi^*)\leq b + \frac{\sqrt{2\delta}\gamma\epsilon_{C}^{\pi^*}}{(1-\gamma)^2}
\end{equation}
where $\epsilon_{C}^{\pi^*}=\max_s\left|\E_{a\sim\pi}[A_{C}^{\polk}(s,a)]\right|$.

\subsection{Approximating the Optimal Update Policy}
 
When solving Problem (\ref{eq:FOCOPS_obj}-\ref{eq:FOCOPS_tr_const}), we allow $\pi$ to be in the set of all stationary policies $\Pi$ thus the resulting $\pi^*$ is not necessarily in the parameterized policy space $\Pi_{\theta}$ and we may no longer be able to evaluate or sample from $\pi^*$. Therefore in the second step we project the optimal update policy back into the parameterized policy space by minimizing the loss function:
\begin{equation}\label{eq:loss_func}
    \cL(\theta) = \E_{s\sim d^{\polk}}\left[\KL{\pol}{\pi^*}[s]\right]
\end{equation}
Here $\pol\in\Pi_{\theta}$ is some projected policy which we will use to approximate the optimal update policy. We can use first-order methods to minimize this loss function where we make use of the following result:
\begin{restatable}{corollary}{lossgrad}\label{corollary:loss_grad}
The gradient of $\cL(\theta)$ takes the form
\begin{equation}\label{eq:loss_grad}
    \grad\cL(\theta) = \E_{s\sim d^{\polk}}\left[\grad\KL{\pol}{\pi^*}[s]\right]
\end{equation}
where
\begin{small}
\begin{equation}\label{eq:loss_grad_s}
    \grad\KL{\pol}{\pi^*}[s] = \grad\KL{\pol}{\polk}[s] 
    - \frac{1}{\lambda}\E_{a\sim\polk}\left[\frac{\grad\pol(a|s)}{\polk(a|s)}\left(\advk(s, a) - \nu A^{\polk}_{C}(s,a) \right)\right]
\end{equation}
\end{small}
\end{restatable}
\begin{proof}
See Appendix \ref{append:loss_grad} of the supplementary materials.
\end{proof}
Note that \eqref{eq:loss_grad} can be estimated by sampling from the trajectories generated by policy $\polk$ which allows us to train our policy using stochastic gradients.

Corollary \ref{corollary:loss_grad} provides an outline for our algorithm. At every iteration we begin with a policy $\polk$, which we use to run trajectories and gather data. We use that data and \eqref{eq:dual} to first estimate $\lambda$ and $\nu$. We then draw a minibatch from the data to estimate $\grad\cL(\theta)$ given in Corollary \ref{corollary:loss_grad}. After taking a gradient step using Equation \eqref{eq:loss_grad}, we draw another minibatch and repeat the process. 

\subsection{Practical Implementation}

Solving the dual problem \eqref{eq:dual} is computationally impractical for large state/action spaces as it requires calculating the partition function $Z_{\lambda,\nu}(s)$ which often involves evaluating a high-dimensional integral or sum. Furthermore, $\lambda$ and $\nu$ depends on $k$ and should be adapted at every iteration. 

We note that as $\lambda\to 0$, $\pi^*$ approaches a greedy policy; as $\lambda$ increases, the policy becomes more exploratory. We also note that $\lambda$ is similar to the temperature term used in maximum entropy reinforcement learning \citep{ziebart2008maximum}, which has been shown to produce reasonable results when kept fixed during training \citep{schulman2017equivalence,haarnoja2018soft}. In practice, we found that a fixed $\lambda$ found through hyperparameter sweeps provides good results. However $\nu$ needs to be continuously adapted during training so as to ensure cost constraint satisfaction. Here we appeal to an intuitive heuristic for determining $\nu$ based on primal-dual gradient methods \citep{bertsekas2014constrained}. Recall that by strong duality, the optimal $\lambda^*$ and $\nu^*$ minimizes the dual function \eqref{eq:dual} which we will denote by $L(\pi^*,\lambda,\nu)$. We can therefore apply gradient descent w.r.t. $\nu$ to minimize $L(\pi^*,\lambda,\nu)$. We can show that
\begin{restatable}{corollary}{lagrangegrad}\label{corollary:lagrange_grad}
The derivative of $L(\pi^*,\lambda,\nu)$ w.r.t. $\nu$ is
\begin{equation}\label{eq:dual_grad}
    \frac{\partial L(\pi^*,\lambda,\nu)}{\partial \nu} 
    = \Tilde{b} - \E_{\substack{s\sim\dpolk\\ a\sim\pi^*}}[A^{\polk}(s,a)]
\end{equation}
\end{restatable}
\begin{proof}
See Appendix \ref{append:lagrange_grad} of the supplementary materials.
\end{proof}
The last term in the gradient expression in Equation \eqref{eq:dual_grad} cannot be evaluated since we do not have access to $\pi^*$. However since $\polk$ and $\pi^*$ are 'close' (by constraint \eqref{eq:FOCOPS_tr_const}), it is reasonable to assume that $E_{s\sim\dpolk,a\sim\pi^*}[A^{\polk}(s,a)]\approx E_{s\sim\dpolk,a\sim\polk}[A^{\polk}(s,a)]=0$. In practice we find that this term can be set to zero which gives the update term:
\begin{equation}\label{eq:nu_update}
    \nu\gets\proj_{\nu}\left[\nu -
    \alpha(b - J_C(\polk))\right]
\end{equation}
where $\alpha$ is the step size, here we have incorporated the discount term $(1-\gamma)$ in $\Tilde{b}$ into the step size. The projection operator $\proj_{\nu}$ projects $\nu$ back into the interval $[0,\nu_{\max}]$ where $\nu_{\max}$ is chosen so that $\nu$ does not become too large. However we will show in later sections that FOCOPS is generally insensitive to the choice of $\nu_{\max}$ and setting $\nu_{\max}=+\infty$ does not appear to greatly reduce performance. Practically, $J_C(\polk)$ can be estimated via Monte Carlo methods using trajectories collected from $\polk$. We note that the update rule in Equation \eqref{eq:nu_update} is similar in to the update rule introduced in \citet{chow2017risk}. We recall that in \eqref{eq:FOCOPS_opt}, $\nu$ acts as a cost penalty term where increasing $\nu$ makes it less likely for state-action pairs with higher costs to be sampled by $\pi^*$. Hence in this regard, the update rule in \eqref{eq:nu_update} is intuitive in that it increases $\nu$ if $J_C(\polk)>b$ (i.e. the cost constraint is violated for $\polk$) and decreases $\nu$ otherwise. Using the update rule \eqref{eq:nu_update}, we can then perform one update step on $\nu$ before updating the policy parameters $\theta$.

Our method is a first-order method, so the approximations that we make is only accurate near the initial condition (i.e. $\pi_{\theta}=\pi_{\theta_k}$). In order to better enforce this we also add to \eqref{eq:loss_grad} a per-state acceptance indicator function $I(s_j):=\1_{\KL{\pol}{\polk}[s_j]\leq\delta}$. This way sampled states whose $\KL{\pol}{\polk}[s]$ is too large are rejected from the gradient update. The resulting sample gradient update term is
\begin{equation}\label{eq:grad_est}
    \hat{\nabla}_{\theta} \cL(\theta)
    \approx \frac{1}{N}\sum_{j=1}^N\bigg[\grad\KL{\pol}{\polk}[s_j] -\frac{1}{\lambda}\frac{\grad\pol(a_j|s_j)}{\polk(a_j|s_j)}\bigg(\hat{A}(s_j, a_j) - \nu\hat{A}_{C}(s_j,a_j) \bigg)\bigg]I(s_j).
\end{equation}
Here $N$ is the number of samples we collected using policy $\polk$, $\hat{A}$ and $\hat{A}_{C}$ are estimates of the advantage functions (for the return and cost) obtained from critic networks. We estimate the advantage functions using the Generalized Advantage Estimator (GAE) \citep{schulman2016high}. We can then apply stochastic gradient descent using Equation \eqref{eq:grad_est}. During training, we use the early stopping criteria $\frac{1}{N}\sum_{i=1}^N\KL{\pi_{\theta}}{\polk}[s_i]>\delta$ which helps prevent trust region constraint violation for the new updated policy. We update the parameters for the value net by minimizing the Mean Square Error (MSE) of the value net output and some target value (which can be estimated via Monte Carlo or bootstrap estimates of the return). We emphasize again that FOCOPS only requires first order methods (gradient descent) and is thus extremely simple to implement. 

Algorithm \ref{alg:FOCOPS_outline} presents a summary of the FOCOPS algorithm. A more detailed pseudocode is provided in Appendix \ref{append:pseudocode} of the supplementary materials.

\begin{algorithm}[h]
\algtext*{End}
\begin{algorithmic}[1]
\caption{FOCOPS Outline}
\label{alg:FOCOPS_outline}
\INPUT{Policy network $\pi_{\theta_0}$, Value networks $V_{\phi_0}$, $V_{\psi_0}^{C}$.} 
\While{Stopping criteria not met}
    \State Generate trajectories $\tau\sim\polk$.
    \State Estimate $C$-returns and advantage functions.
	\State Update $\nu$ using Equation \eqref{eq:nu_update}.
	\For{$K$ epochs}
	    \For{each minibatch}
	        \State Update value networks by minimizing MSE of $V_{\phi_k}$, $V^{\text{target}}_{\phi_k}$ and  $V_{\psi_k}^{C}$, $V^{C,\text{target}}_{\psi_k}$.
	        \State Update policy network using Equation \eqref{eq:grad_est}
	    \EndFor
	   \If{$\frac{1}{N}\sum_{j=1}^N\KL{\pol}{\polk}[s_j]>\delta$}
    	        \State Break out of inner loop
    	\EndIf
	\EndFor
\EndWhile
\end{algorithmic}
\end{algorithm}

\section{Experiments}\label{sec:experiments}
We designed two different sets of experiments to test the efficacy of the FOCOPS algorithm. In the first set of experiments, we train different robotic agents to move along a straight line or a two dimensional plane, but the speed of the robot is constrained for safety purposes. The second set of experiments is inspired by the Circle experiments from \citet{achiam2017constrained}.  Both sets of experiments are implemented using the OpenAI Gym API \citep{brockman2016openai} for the MuJoCo physical simulator \citep{todorov2012mujoco}. Implementation details for all experiments can be found in the supplementary material.

In addition to the CPO algorithm, we are also including for comparison two algorithms based on Lagrangian methods \citep{bertsekas1997nonlinear}, which uses adaptive penalty coefficients to enforce
constraints. For an objective function $f(\theta)$ and constraint $g(\theta)\leq 0$, the Lagrangian method solves max-min optimization problem $\max_{\theta}\min_{\nu\geq 0} (f(\theta) - \nu g(\theta))$. These methods first perform gradient ascent on $\theta$, and then gradient descent on $\nu$. \cite{chow2019lyapunov} and  \cite{ray2019benchmarking} combined Lagrangian method with PPO \citep{schulman2017proximal} and TRPO \citep{schulman2015trust} to form the PPO Lagrangian and TRPO Lagrangian algorithms, which we will subsequently abbreviate as PPO-L and TRPO-L respectively. Details for these two algorithms can be found in the supplementary material.

\subsection{Robots with Speed Limit}

\begin{figure}[h]
    \centering
    \includegraphics[width=\linewidth]{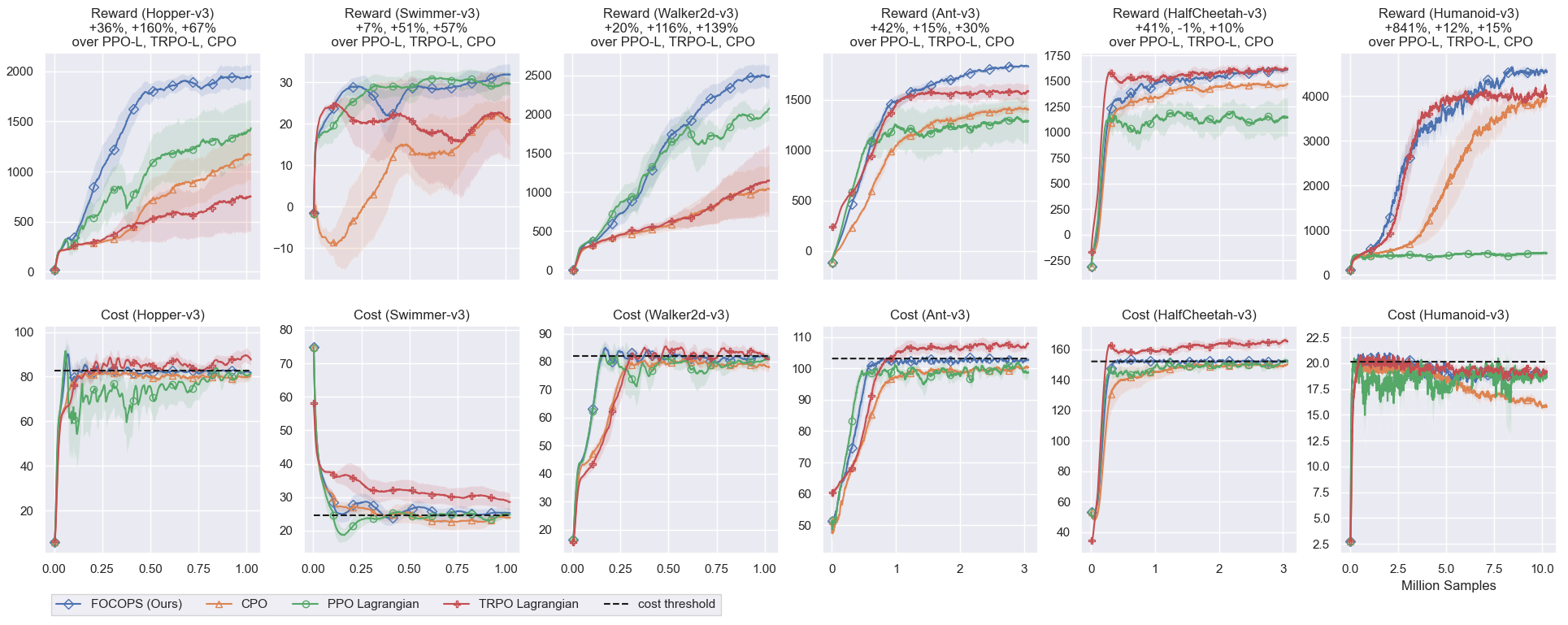}
    \caption{Learning curves for robots with speed limit tasks. The $x$-axis represent the number of samples used and the $y$-axis represent the average total reward/cost return of the last 100 episodes. The solid line represent the mean of 1000 bootstrap samples over 10 random seeds. The shaded regions represent the bootstrap normal $95\%$ confidence interval. FOCOPS  consistently enforce approximate constraint satisfaction while having a higher performance on five out of the six tasks.}
    \label{fig:mujoco_plot}
\end{figure}

We consider six MuJoCo environments where we attempt to train a robotic agent to walk. However we impose a speed limit on our environments. The cost thresholds are calculated using 50\% of the speed attained by an unconstrained PPO agent after training for a million samples (Details can be found in Appendix \ref{append:robots}).

Figure \ref{fig:mujoco_plot} shows that FOCOPS outperforms other baselines in terms of reward on most tasks while enforcing the cost constraint. In theory, FOCOPS assumes that the initial policy is feasible. This assumption is violated in the Swimmer-v3 environment. However in practice, the gradient update term increases the dual variable associated with the cost when the cost constraint is violated, this would result in a feasible policy after a certain number of iterations. We observed that this is indeed the case with the swimmer environment (and similarly the AntCircle environment in the next section). Note also that Lagrangian methods outperform CPO on several environments in terms of reward, this is consistent with the observation made by \citet{ray2019benchmarking} and \citet{stooke2020responsive}. However on most tasks TRPO-L does not appear to consistently maintain constraint satisfaction during training. For example on HalfCheetah-v3, even though TRPO-L outperforms FOCOPS in terms of total return, it violates the cost constraint by nearly 9\%. PPO-L is shown to do well on simpler tasks but performance deteriorates drastically on the more challenging environments (Ant-v3, HalfCheetah-v3, and Humanoid-v3), this is in contrast to FOCOPS which perform particularly well on these set of tasks. In Table \ref{tab:mujoco_velocity_results} we summarized the performance of all four algorithms.
\vskip 0.1in
\begin{table}[h]
    \centering
    \caption{Bootstrap mean and normal $95\%$ confidence interval with 1000 bootstrap samples over 10 random seeds of reward/cost return after training on robot with speed limit environments. Cost thresholds are in brackets under the environment names.}
    \vskip 0.1in
    \begin{adjustbox}{width=\textwidth}
    \begin{tabular}{*6c}
    \toprule
    Environment   & {} & PPO-L  & TRPO-L & CPO & \textbf{FOCOPS} \\
    \midrule
    Ant-v3 &Reward  & $1291.4\pm 216.4$ & $1585.7\pm 77.5$ & $1406.0\pm 46.6$ & $\boldsymbol{1830.0\pm 22.6}$ \\
    (103.12) &Cost  & $98.78\pm 1.77$ & $107.82\pm 1.16$ & $100.25\pm 0.67$ & $102.75\pm 1.08$ \\
    \midrule
    HalfCheetah-v3 &Reward  & $1141.3\pm 192.4$ & $\boldsymbol{1621.59\pm 39.4}$ & $1470.8\pm 40.0$ & $1612.2\pm 25.9$ \\
    (151.99) &Cost  & $151.53\pm 1.88$ & $164.93\pm 2.43$ & $150.05\pm 1.40$ & $152.36\pm 1.55$ \\
    \midrule
    Hopper-v3 &Reward  & $1433.8\pm 313.3$ & $750.3\pm 355.3$ & $1167.1\pm 257.6$ & $\boldsymbol{1953.4\pm 127.3}$ \\
    (82.75) &Cost  & $81.29\pm 2.34$ & $87.57\pm 3.48$ & $80.39\pm 1.39$ & $81.84\pm 0.92$ \\
    \midrule
    Humanoid-v3 &Reward  & $471.3\pm 49.0$ & $4062.4\pm 113.3$ & $3952.7\pm 174.4$ & $\boldsymbol{4529.7\pm 86.2}$ \\
    (20.14) &Cost  & $18.89\pm 0.77$ & $19.23\pm 0.76$ & $15.83\pm 0.41$ & $18.63\pm 0.37$ \\
    \midrule
    Swimmer-v3 &Reward  & $29.73\pm 3.13$ & $21.15\pm 9.56$ & $20.31\pm 6.01$ & $\boldsymbol{31.94\pm 2.60}$ \\
    (24.52) &Cost  & $24.72\pm 0.85$ & $28.57\pm 2.68$ & $23.88\pm 0.64$ & $25.29\pm 1.49$ \\
    \midrule
    Walker2d-v3 &Reward  & $2074.4\pm 155.7$ & $1153.1\pm 473.3$ & $1040.0\pm 303.3$ & $\boldsymbol{2485.9\pm 158.3}$ \\
    (81.89) &Cost  & $81.7\pm 1.14$ & $80.79\pm 2.13$ & $78.12\pm 1.78$ & $81.27\pm 1.33$ \\
    \bottomrule
    \end{tabular}
    \end{adjustbox}
    \label{tab:mujoco_velocity_results}
\end{table}

\subsection{Circle Tasks}

For these tasks, we use the same exact geometric setting, reward, and cost constraint function as \citet{achiam2017constrained}, a geometric illustration of the task and details on the reward/cost functions can be found in Appendix \ref{append:circle_experiments} of the supplementary materials. The goal of the agents is to move along the circumference of a circle while remaining within a safe region smaller than the radius of the circle. 

Similar to the previous tasks, we provide learning curves (Figure \ref{fig:circle_plot}) and numerical summaries (Table \ref{tab:circle_results}) of the experiments. We also plotted an unconstrained PPO agent for comparison. On these tasks, all four approaches are able to approximately enforce cost constraint satisfaction (set at 50), but FOCOPS does so while having a higher performance. Note for both tasks, the $95\%$ confidence interval for FOCOPS lies above the confidence intervals for all other algorithms, this is strong indication that FOCOPS outperforms the other three algorithms on these particular tasks.
\begin{figure}[H]
    \centering
    \includegraphics[width=\textwidth]{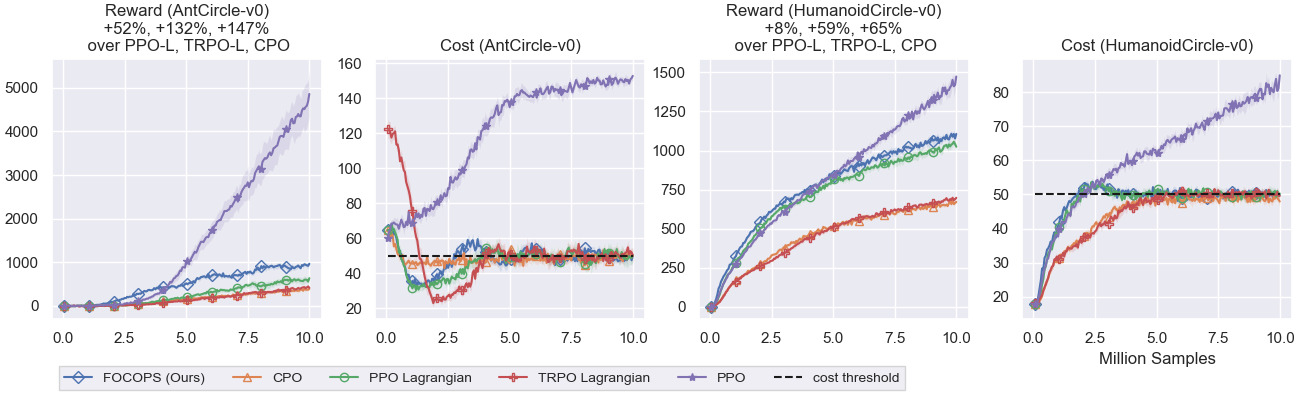}
    \caption{Comparing reward and cost returns on circle Tasks. The $x$-axis represent the number of samples used and the $y$-axis represent the average total reward/cost return of the last 100 episodes. The solid line represent the mean of 1000 bootstrap samples over 10 random seeds. The shaded regions represent the bootstrap normal $95\%$ confidence interval. An unconstrained PPO agent is also plotted for comparison.}
    \label{fig:circle_plot}
\end{figure}
\vskip 0.1in
\begin{table}[h]
    \centering
    \caption{Bootstrap mean and normal $95\%$ confidence interval with 1000 bootstrap samples over 10 random seeds of reward/cost return after training on circle environments for 10 million samples. Cost thresholds are in brackets under the environment names.}
    \vskip 0.1in
    \begin{adjustbox}{width=\textwidth}
    \begin{tabular}{*6c}
    \toprule
    Environment & {} & PPO-L  & TRPO-L & CPO & \textbf{FOCOPS} \\
    \midrule
    Ant-Circle &Reward  & $637.4\pm 88.2$ & $416.7\pm 42.1$ & $390.9\pm 43.9$ & $\boldsymbol{965.9\pm 46.2}$ \\
    (50.0) &Cost  & $50.4\pm 4.4$ & $50.4\pm 3.9$ & $50.0\pm 3.5$ & $49.9\pm 2.2$ \\
    \midrule
    Humanoid-Circle &Reward  & $1024.5\pm 23.4$ & $697.5\pm 14.0$ & $671.0\pm 12.5$ & $\boldsymbol{1106.1\pm 32.2}$ \\
    (50.0) &Cost  & $50.3\pm 0.59$ & $49.6\pm 0.96$ & $47.9\pm 1.5$ & $49.9\pm 0.8$ \\
    \bottomrule
    \end{tabular}
    \end{adjustbox}
    \label{tab:circle_results}
\end{table}

\subsection{Generalization Analysis}
In supervised learning, the standard approach is to use separate datasets for training, validation, and testing where we can then use some evaluation metric on the test set to determine how well an algorithm generalizes over unseen data. However such a scheme is not suitable for reinforcement learning.

To similarly evaluate our reinforcement learning agent, we first trained our agent on a fixed random seed. We then tested the trained agent on ten unseen random seeds \citep{pineau18reproducible}. We found that with the exception of Hopper-v3, FOCOPS outperformed every other constrained algorithm on all robots with speed limit environments. Detailed analysis of the generalization results are provided in Appendix \ref{append:generalization} of the supplementary materials.

\subsection{Sensitivity Analysis}\label{sec:sensitivity}
The Lagrange multipliers play a key role in the FOCOPS algorithms, in this section we explore the sensitivity of the hyperparameters $\lambda$ and $\nu_{\max}$. We find that the performance of FOCOPS is largely insensitive to the choice of these hyperparamters. To demonstrate this, we conducted a series of experiments on the robots with speed limit tasks.

The hyperparameter $\nu_{\max}$ was selected via hyperparameter sweep on the set $\{1,2,3,5,10,+\infty\}$. However we found that FOCOPS is not sensitive to the choice of $\nu_{\max}$ where setting $\nu_{\max}=+\infty$ only leads to an average $0.3\%$ degradation in performance compared to the optimal $\nu_{\max}=2$. Similarly we tested the performance of FOCOPS against different values of $\lambda$ where for other values of $\lambda$, FOCOPS performs on average $7.4\%$ worse compared to the optimal $\lambda=1.5$. See Appendix \ref{append:sensitive_analysis} of the supplementary materials for more details.

\section{Discussion}
We introduced FOCOPS\textemdash a simple first-order approach for training RL agents with safety constraints. FOCOPS is theoretically motivated and is shown to empirically outperform more complex second-order methods. FOCOPS is also easy to implement. We believe in the value of simplicity as it makes RL more accessible to researchers in other fields who wish to apply such methods in their own work. Our results indicate that constrained RL is an effective approach for addressing RL safety and can be efficiently solved using our two step approach. 

There are a number of promising avenues for future work: such as incorporating off-policy data; studying how our two-step approach can deal with different types of constraints such as sample path constraints \citep{ross1989markov,ross1991multichain}, or safety constraints expressed in other more natural forms such as human preferences \citep{christiano2017deep} or natural language \citep{luketina2019survey}.

\section{Broader Impact}
Safety is a critical element in real-world applications of RL. We argue in this paper that scalar reward signals alone is often insufficient in motivating the agent to avoid harmful behavior. An RL systems designer needs to carefully balance how much to encourage desirable behavior and how much to penalize unsafe behavior where too much penalty could prevent the agent from sufficient exploration and too little could lead to hazardous consequences. This could be extremely difficult in practice. Constraints are a more natural way of quantifying safety requirements and we advocate for researchers to consider including constraints in safety-critical RL systems.

\begin{ack}
The authors would like to express our appreciation for the technical support provided by the NYU Shanghai High Performance Computing (HPC) administrator Zhiguo Qi and the HPC team at NYU. We would also like to thank Joshua Achiam for his insightful comments and for making his implementation of the CPO algorithm publicly available.
\end{ack}

\bibliographystyle{abbrvnat}
\bibliography{focops}

\begin{thebibliography}{49}
\providecommand{\natexlab}[1]{#1}
\providecommand{\url}[1]{\texttt{#1}}
\expandafter\ifx\csname urlstyle\endcsname\relax
  \providecommand{\doi}[1]{doi: #1}\else
  \providecommand{\doi}{doi: \begingroup \urlstyle{rm}\Url}\fi

\bibitem[Abdolmaleki et~al.(2018)Abdolmaleki, Springenberg, Tassa, Munos,
  Heess, and Riedmiller]{abdolmaleki2018maximum}
A.~Abdolmaleki, J.~T. Springenberg, Y.~Tassa, R.~Munos, N.~Heess, and
  M.~Riedmiller.
\newblock Maximum a posteriori policy optimisation.
\newblock In \emph{International Conference on Learning Representations}, 2018.

\bibitem[Achiam et~al.(2017)Achiam, Held, Tamar, and
  Abbeel]{achiam2017constrained}
J.~Achiam, D.~Held, A.~Tamar, and P.~Abbeel.
\newblock Constrained policy optimization.
\newblock In \emph{Proceedings of the 34th International Conference on Machine
  Learning-Volume 70}, pages 22--31. JMLR. org, 2017.

\bibitem[Altman(1999)]{altman1999constrained}
E.~Altman.
\newblock \emph{Constrained Markov decision processes}, volume~7.
\newblock CRC Press, 1999.

\bibitem[Amodei et~al.(2016)Amodei, Olah, Steinhardt, Christiano, Schulman, and
  Man{\'e}]{amodei2016concrete}
D.~Amodei, C.~Olah, J.~Steinhardt, P.~Christiano, J.~Schulman, and D.~Man{\'e}.
\newblock Concrete problems in ai safety.
\newblock \emph{arXiv preprint arXiv:1606.06565}, 2016.

\bibitem[Bertsekas(1997)]{bertsekas1997nonlinear}
D.~P. Bertsekas.
\newblock Nonlinear programming.
\newblock \emph{Journal of the Operational Research Society}, 48\penalty0
  (3):\penalty0 334--334, 1997.

\bibitem[Bertsekas(2014)]{bertsekas2014constrained}
D.~P. Bertsekas.
\newblock \emph{Constrained optimization and Lagrange multiplier methods}.
\newblock Academic press, 2014.

\bibitem[Beutler and Ross(1985)]{beutler1985optimal}
F.~J. Beutler and K.~W. Ross.
\newblock Optimal policies for controlled markov chains with a constraint.
\newblock \emph{Journal of mathematical analysis and applications},
  112\penalty0 (1):\penalty0 236--252, 1985.

\bibitem[Beutler and Ross(1986)]{beutler1986time}
F.~J. Beutler and K.~W. Ross.
\newblock Time-average optimal constrained semi-markov decision processes.
\newblock \emph{Advances in Applied Probability}, 18\penalty0 (2):\penalty0
  341--359, 1986.

\bibitem[Boyd and Vandenberghe(2004)]{boyd2004convex}
S.~Boyd and L.~Vandenberghe.
\newblock \emph{Convex optimization}.
\newblock Cambridge university press, 2004.

\bibitem[Brockman et~al.(2016)Brockman, Cheung, Pettersson, Schneider,
  Schulman, Tang, and Zaremba]{brockman2016openai}
G.~Brockman, V.~Cheung, L.~Pettersson, J.~Schneider, J.~Schulman, J.~Tang, and
  W.~Zaremba.
\newblock Openai gym, 2016.

\bibitem[Chow et~al.(2017)Chow, Ghavamzadeh, Janson, and Pavone]{chow2017risk}
Y.~Chow, M.~Ghavamzadeh, L.~Janson, and M.~Pavone.
\newblock Risk-constrained reinforcement learning with percentile risk
  criteria.
\newblock \emph{The Journal of Machine Learning Research}, 18\penalty0
  (1):\penalty0 6070--6120, 2017.

\bibitem[Chow et~al.(2018)Chow, Nachum, Duenez-Guzman, and
  Ghavamzadeh]{chow2018lyapunov}
Y.~Chow, O.~Nachum, E.~Duenez-Guzman, and M.~Ghavamzadeh.
\newblock A lyapunov-based approach to safe reinforcement learning.
\newblock In \emph{Advances in neural information processing systems}, pages
  8092--8101, 2018.

\bibitem[Chow et~al.(2019)Chow, Nachum, Faust, Ghavamzadeh, and
  Duenez-Guzman]{chow2019lyapunov}
Y.~Chow, O.~Nachum, A.~Faust, M.~Ghavamzadeh, and E.~Duenez-Guzman.
\newblock Lyapunov-based safe policy optimization for continuous control.
\newblock \emph{arXiv preprint arXiv:1901.10031}, 2019.

\bibitem[Christiano et~al.(2017)Christiano, Leike, Brown, Martic, Legg, and
  Amodei]{christiano2017deep}
P.~F. Christiano, J.~Leike, T.~Brown, M.~Martic, S.~Legg, and D.~Amodei.
\newblock Deep reinforcement learning from human preferences.
\newblock In \emph{Advances in Neural Information Processing Systems}, pages
  4299--4307, 2017.

\bibitem[Duan et~al.(2016)Duan, Chen, Houthooft, Schulman, and
  Abbeel]{duan2016benchmarking}
Y.~Duan, X.~Chen, R.~Houthooft, J.~Schulman, and P.~Abbeel.
\newblock Benchmarking deep reinforcement learning for continuous control.
\newblock In \emph{International Conference on Machine Learning}, pages
  1329--1338, 2016.

\bibitem[Haarnoja et~al.(2018)Haarnoja, Zhou, Abbeel, and
  Levine]{haarnoja2018soft}
T.~Haarnoja, A.~Zhou, P.~Abbeel, and S.~Levine.
\newblock Soft actor-critic: Off-policy maximum entropy deep reinforcement
  learning with a stochastic actor.
\newblock \emph{International Conference on Machine Learning}, 2018.

\bibitem[Kakade and Langford(2002)]{kakade2002approximately}
S.~Kakade and J.~Langford.
\newblock Approximately optimal approximate reinforcement learning.
\newblock In \emph{International Conference on Machine Learning}, volume~2,
  pages 267--274, 2002.

\bibitem[Kallenberg(1983)]{kallenberg1983linear}
L.~Kallenberg.
\newblock \emph{Linear Programming and Finite Markovian Control Problems}.
\newblock Centrum Voor Wiskunde en Informatica, 1983.

\bibitem[Luketina et~al.(2019)Luketina, Nardelli, Farquhar, Foerster, Andreas,
  Grefenstette, Whiteson, and Rockt{\"a}schel]{luketina2019survey}
J.~Luketina, N.~Nardelli, G.~Farquhar, J.~Foerster, J.~Andreas,
  E.~Grefenstette, S.~Whiteson, and T.~Rockt{\"a}schel.
\newblock A survey of reinforcement learning informed by natural language.
\newblock \emph{arXiv preprint arXiv:1906.03926}, 2019.

\bibitem[Mnih et~al.(2013)Mnih, Kavukcuoglu, Silver, Graves, Antonoglou,
  Wierstra, and Riedmiller]{mnih2013playing}
V.~Mnih, K.~Kavukcuoglu, D.~Silver, A.~Graves, I.~Antonoglou, D.~Wierstra, and
  M.~Riedmiller.
\newblock Playing atari with deep reinforcement learning.
\newblock \emph{NIPS Deep Learning Workshop}, 2013.

\bibitem[Mnih et~al.(2016)Mnih, Badia, Mirza, Graves, Lillicrap, Harley,
  Silver, and Kavukcuoglu]{mnih2016asynchronous}
V.~Mnih, A.~P. Badia, M.~Mirza, A.~Graves, T.~Lillicrap, T.~Harley, D.~Silver,
  and K.~Kavukcuoglu.
\newblock Asynchronous methods for deep reinforcement learning.
\newblock In \emph{International Conference on Machine Learning}, pages
  1928--1937, 2016.

\bibitem[Montgomery and Levine(2016)]{montgomery2016guided}
W.~H. Montgomery and S.~Levine.
\newblock Guided policy search via approximate mirror descent.
\newblock In \emph{Advances in Neural Information Processing Systems}, pages
  4008--4016, 2016.

\bibitem[Peters and Schaal(2008)]{peters2008reinforcement}
J.~Peters and S.~Schaal.
\newblock Reinforcement learning of motor skills with policy gradients.
\newblock \emph{Neural networks}, 21\penalty0 (4):\penalty0 682--697, 2008.

\bibitem[Peters et~al.(2010)Peters, Mulling, and Altun]{peters2010relative}
J.~Peters, K.~Mulling, and Y.~Altun.
\newblock Relative entropy policy search.
\newblock In \emph{Twenty-Fourth AAAI Conference on Artificial Intelligence},
  2010.

\bibitem[Pineau(2018)]{pineau18reproducible}
J.~Pineau.
\newblock {NeurIPS 2018 Invited Talk: {R}eproducible, {R}eusable, and {R}obust
  {R}einforcement {L}earning}, 2018.
\newblock URL:
  \url{https://media.neurips.cc/Conferences/NIPS2018/Slides/jpineau-NeurIPS-dec18-fb.pdf}.
  Last visited on 2020/05/28.

\bibitem[Pirotta et~al.(2013)Pirotta, Restelli, Pecorino, and
  Calandriello]{pirotta2013safe}
M.~Pirotta, M.~Restelli, A.~Pecorino, and D.~Calandriello.
\newblock Safe policy iteration.
\newblock In \emph{International Conference on Machine Learning}, pages
  307--315, 2013.

\bibitem[Ray et~al.(2019)Ray, Achiam, and Amodei]{ray2019benchmarking}
A.~Ray, J.~Achiam, and D.~Amodei.
\newblock {Benchmarking Safe Exploration in Deep Reinforcement Learning}.
\newblock \emph{arXiv preprint arXiv:1910.01708}, 2019.

\bibitem[Ross(1985)]{ross1985constrained}
K.~W. Ross.
\newblock Constrained markov decision processes with queueing applications.
\newblock \emph{Dissertation Abstracts International Part B: Science and
  Engineering[DISS. ABST. INT. PT. B- SCI. \& ENG.],}, 46\penalty0 (4), 1985.

\bibitem[Ross(1989)]{ross1989randomized}
K.~W. Ross.
\newblock Randomized and past-dependent policies for markov decision processes
  with multiple constraints.
\newblock \emph{Operations Research}, 37\penalty0 (3):\penalty0 474--477, 1989.

\bibitem[Ross and Varadarajan(1989)]{ross1989markov}
K.~W. Ross and R.~Varadarajan.
\newblock Markov decision processes with sample path constraints: the
  communicating case.
\newblock \emph{Operations Research}, 37\penalty0 (5):\penalty0 780--790, 1989.

\bibitem[Ross and Varadarajan(1991)]{ross1991multichain}
K.~W. Ross and R.~Varadarajan.
\newblock Multichain markov decision processes with a sample path constraint: A
  decomposition approach.
\newblock \emph{Mathematics of Operations Research}, 16\penalty0 (1):\penalty0
  195--207, 1991.

\bibitem[Schaul et~al.(2015)Schaul, Quan, Antonoglou, and
  Silver]{schaul2015prioritized}
T.~Schaul, J.~Quan, I.~Antonoglou, and D.~Silver.
\newblock Prioritized experience replay.
\newblock \emph{arXiv preprint arXiv:1511.05952}, 2015.

\bibitem[Schulman et~al.(2015)Schulman, Levine, Abbeel, Jordan, and
  Moritz]{schulman2015trust}
J.~Schulman, S.~Levine, P.~Abbeel, M.~Jordan, and P.~Moritz.
\newblock Trust region policy optimization.
\newblock In \emph{International Conference on Machine Learning}, pages
  1889--1897, 2015.

\bibitem[Schulman et~al.(2016)Schulman, Moritz, Levine, Jordan, and
  Abbeel]{schulman2016high}
J.~Schulman, P.~Moritz, S.~Levine, M.~Jordan, and P.~Abbeel.
\newblock High-dimensional continuous control using generalized advantage
  estimation.
\newblock In \emph{International Conference on Learning Representations}, 2016.

\bibitem[Schulman et~al.(2017{\natexlab{a}})Schulman, Chen, and
  Abbeel]{schulman2017equivalence}
J.~Schulman, X.~Chen, and P.~Abbeel.
\newblock Equivalence between policy gradients and soft q-learning.
\newblock \emph{arXiv preprint arXiv:1704.06440}, 2017{\natexlab{a}}.

\bibitem[Schulman et~al.(2017{\natexlab{b}})Schulman, Wolski, Dhariwal,
  Radford, and Klimov]{schulman2017proximal}
J.~Schulman, F.~Wolski, P.~Dhariwal, A.~Radford, and O.~Klimov.
\newblock Proximal policy optimization algorithms.
\newblock \emph{arXiv preprint arXiv:1707.06347}, 2017{\natexlab{b}}.

\bibitem[Silver et~al.(2016)Silver, Huang, Maddison, Guez, Sifre, Van
  Den~Driessche, Schrittwieser, Antonoglou, Panneershelvam, Lanctot,
  et~al.]{silver2016mastering}
D.~Silver, A.~Huang, C.~J. Maddison, A.~Guez, L.~Sifre, G.~Van Den~Driessche,
  J.~Schrittwieser, I.~Antonoglou, V.~Panneershelvam, M.~Lanctot, et~al.
\newblock Mastering the game of go with deep neural networks and tree search.
\newblock \emph{nature}, 529\penalty0 (7587):\penalty0 484, 2016.

\bibitem[Silver et~al.(2018)Silver, Hubert, Schrittwieser, Antonoglou, Lai,
  Guez, Lanctot, Sifre, Kumaran, Graepel, et~al.]{silver2018mastering}
D.~Silver, T.~Hubert, J.~Schrittwieser, I.~Antonoglou, M.~Lai, A.~Guez,
  M.~Lanctot, L.~Sifre, D.~Kumaran, T.~Graepel, et~al.
\newblock Mastering chess and shogi by self-play with a general reinforcement
  learning algorithm.
\newblock \emph{Science}, 362\penalty0 (6419):\penalty0 1140--1144, 2018.

\bibitem[Stooke et~al.(2020)Stooke, Achiam, and Abbeel]{stooke2020responsive}
A.~Stooke, J.~Achiam, and P.~Abbeel.
\newblock Responsive safety in reinforcement learning by pid lagrangian
  methods.
\newblock In \emph{International Conference on Machine Learning}, 2020.

\bibitem[Strang(2007)]{strang2007computational}
G.~Strang.
\newblock \emph{Computational science and engineering}.
\newblock Wellesley-Cambridge Press, 2007.

\bibitem[Sutton and Barto(2018)]{sutton2018reinforcement}
R.~S. Sutton and A.~G. Barto.
\newblock \emph{Reinforcement learning: An introduction}.
\newblock MIT press, 2018.

\bibitem[Tessler et~al.(2019)Tessler, Mankowitz, and Mannor]{tessler2018reward}
C.~Tessler, D.~J. Mankowitz, and S.~Mannor.
\newblock Reward constrained policy optimization.
\newblock \emph{International Conference on Learning Representation}, 2019.

\bibitem[Todorov et~al.(2012)Todorov, Erez, and Tassa]{todorov2012mujoco}
E.~Todorov, T.~Erez, and Y.~Tassa.
\newblock Mujoco: A physics engine for model-based control.
\newblock In \emph{2012 IEEE/RSJ International Conference on Intelligent Robots
  and Systems}, pages 5026--5033. IEEE, 2012.

\bibitem[Van~Hasselt et~al.(2016)Van~Hasselt, Guez, and Silver]{van2016deep}
H.~Van~Hasselt, A.~Guez, and D.~Silver.
\newblock Deep reinforcement learning with double q-learning.
\newblock In \emph{AAAI}, volume~2, page~5. Phoenix, AZ, 2016.

\bibitem[Vuong et~al.(2019)Vuong, Zhang, and Ross]{vuong2018supervised}
Q.~Vuong, Y.~Zhang, and K.~W. Ross.
\newblock Supervised policy update for deep reinforcement learning.
\newblock In \emph{International Conference on Learning Representation}, 2019.

\bibitem[Wang et~al.(2017)Wang, Bapst, Heess, Mnih, Munos, Kavukcuoglu, and
  de~Freitas]{wang2017sample}
Z.~Wang, V.~Bapst, N.~Heess, V.~Mnih, R.~Munos, K.~Kavukcuoglu, and
  N.~de~Freitas.
\newblock Sample efficient actor-critic with experience replay.
\newblock \emph{International Conference on Learning Representations}, 2017.

\bibitem[Wu et~al.(2017)Wu, Mansimov, Grosse, Liao, and Ba]{wu2017scalable}
Y.~Wu, E.~Mansimov, R.~B. Grosse, S.~Liao, and J.~Ba.
\newblock Scalable trust-region method for deep reinforcement learning using
  kronecker-factored approximation.
\newblock In \emph{Advances in neural information processing systems}, pages
  5285--5294, 2017.

\bibitem[Yang et~al.(2020)Yang, Rosca, Narasimhan, and
  Ramadge]{yang2020projection}
T.-Y. Yang, J.~Rosca, K.~Narasimhan, and P.~J. Ramadge.
\newblock Projection-based constrained policy optimization.
\newblock In \emph{International Conference on Learning Representation}, 2020.

\bibitem[Ziebart et~al.(2008)Ziebart, Maas, Bagnell, and
  Dey]{ziebart2008maximum}
B.~D. Ziebart, A.~L. Maas, J.~A. Bagnell, and A.~K. Dey.
\newblock Maximum entropy inverse reinforcement learning.
\newblock In \emph{AAAI}, volume~8, pages 1433--1438. Chicago, IL, USA, 2008.

\end{thebibliography}

\newpage

\noindent\rule{\textwidth}{4pt}
\begin{center}
    \LARGE\bf Supplementary Material for First Order Constrained Optimization in Policy Space
\end{center}
\noindent\rule{\textwidth}{1pt}

\begin{appendices}

\section{Proof of Theorem \ref{thm:FOCOPS_opt_pol} }\label{append:thm}
\FOCOPSsol*

\begin{proof}
We will begin by showing that Problem (\ref{eq:FOCOPS_obj}-\ref{eq:FOCOPS_tr_const}) is convex w.r.t. $\pi=\{\pi(a|s):s\in\cS, a\in\cA\}$. First note that the objective function is linear w.r.t. $\pi$. Since $J_{C}(\polk)$ is a constant w.r.t. $\pi$, constraint \eqref{eq:FOCOPS_cost_const} is linear. Constraint \eqref{eq:FOCOPS_tr_const} can be rewritten as $\sum_{s}\dpolk(s)\KL{\pi}{\polk}[s]\leq \delta$, the KL divergence is convex w.r.t. its first argument, therefore constraint \eqref{eq:FOCOPS_tr_const} which is a linear combination of convex functions is also convex. Since $\polk$ satisfies Constraint \eqref{eq:FOCOPS_cost_const} and is also an interior point within the set given by Constraints \eqref{eq:FOCOPS_tr_const} ($\KL{\polk}{\polk} = 0$, and $\delta>0$), therefore Slater's constraint qualification holds, strong duality holds.

We can therefore solve for the optimal value of Problem (\ref{eq:FOCOPS_obj}-\ref{eq:FOCOPS_tr_const}) $p^*$ by solving the corresponding dual problem. Let
\begin{small}
\begin{equation}\label{eq:FOCOPS_lagrange}
        L(\pi,\lambda,\nu)=\lambda\delta+\nu\Tilde{b}+\E_{s\sim d^{\polk}}\bigg[\E_{a\sim\pi(\cdot|s)}[A^{\polk}(s,a)] - \nu\E_{a\sim\pi(\cdot|s)}[A^{\polk}_{C}(s,a)] \ - \lambda\KL{\pi}{\polk}[s] \bigg]
\end{equation}
\end{small}
Therefore,
\begin{equation}\label{eq:min_max}
   p^*=  \max_{\pi\in\Pi}\min_{\lambda,\nu\geq 0}  L(\pi,\lambda,\nu) =\min_{\lambda,\nu\geq 0}\max_{\pi\in\Pi}L(\pi,\lambda,\nu) 
\end{equation}
where we invoked strong duality in the second equality. We note that if $\pi^*,\lambda^*,\nu^*$ are optimal for $\eqref{eq:min_max}$, $\pi^*$ is also optimal for Problem (\ref{eq:FOCOPS_obj}-\ref{eq:FOCOPS_tr_const}) \citep{boyd2004convex}.

Consider the inner maximization problem in \eqref{eq:min_max}, we can decompose this problem into separate problems, one for each $s$. This gives us an optimization problem of the form,
\begin{equation}\label{eq:focops_each_s}
\begin{aligned}
\underset{\pi}{\text{maximize}} \quad
& \E_{a\sim \pi(\cdot|s)} \bigg[A^{\polk}(s,a) - \nu A^{\polk}_{C}(s,a) -\lambda(\log\pi(a|s)-\log\polk(a|s)) \bigg] \\
\text{subject to} \quad
& \sum_a \pi(a|s) = 1 \\
& \pi(a|s)\geq 0\quad\text{for all }a\in\cA
\end{aligned}   
\end{equation}
which is equivalent to the inner maximization problem in \eqref{eq:min_max}. This is clearly a convex optimization problem which we can solve using a simple Lagrangian argument. We can write the Lagrangian of \eqref{eq:focops_each_s} as
\begin{equation}
 G(\pi) = \sum_a \pi(a|s)\bigg[A^{\polk}(s,a) - \nu A^{\polk}_{C}(s,a) -\lambda(\log\pi(a|s)-\log\polk(a|s))+\zeta \bigg]  - 1
\end{equation}
where $\zeta>0$ is the Lagrange multiplier associated with the constraint $\sum_a\pi(a|s)=1$. Differentiating $G(\pi)$ w.r.t. $\pi(a|s)$ for some $a$:
\begin{equation}\label{eq:lagrange_grad}
    \frac{\partial G}{\partial \pi(a|s)}= A^{\polk}(s,a) - \nu A^{\polk}_{C}(s,a)-\lambda(\log\pi(a|s)+1-\log\polk(a|s))+\zeta
\end{equation}
Setting \eqref{eq:lagrange_grad} to zero and rearranging the term, we obtain
\begin{equation}
    \pi(a|s) = \polk(a|s)\exp\left(\frac{1}{\lambda}\left(A^{\polk}(s,a)-\nu A^{\polk}_{C}(s,a)\right)+\frac{\zeta}{\lambda}+1\right)
\end{equation}
We chose $\zeta$ so that $\sum_a\pi(a|s)=1$ and rewrite $\zeta/\lambda+1$ as $Z_{\lambda,\nu}(s)$. We find that the optimal solution $\pi^*$ to \eqref{eq:focops_each_s} takes the form \begin{equation*}
        \pi^{*}(a|s) = \frac{\polk(a|s)}{Z_{\lambda,\nu}(s)}\exp\left(\frac{1}{\lambda}\left(A^{\polk}(s,a)-\nu A^{\polk}_{C}(s,a) \right)\right)
\end{equation*}
Plugging $\pi^*$ back into Equation \ref{eq:min_max} gives us
\begin{equation*}
    \begin{aligned}
        p^* &= \min_{\lambda,\nu\geq 0}\lambda\delta+\nu\Tilde{b}+\E_{\substack{s\sim\dpolk\\ a\sim\pi^*}}[A^{\polk}(s,a)-\nu A_{C}^{\polk}(s,a) -\lambda(\log\pi^*(a|s)-\log\polk(a|s))] \\
        &= \min_{\lambda,\nu\geq 0}\lambda\delta+\nu\Tilde{b}+\E_{\substack{s\sim\dpolk\\ a\sim\pi^*}}[A^{\polk}(s,a)-\nu A_{C}^{\polk}(s,a) -\lambda(\log\polk(a|s) - \log Z_{\lambda,\nu}(s) \\ &+\frac{1}{\lambda}(A^{\polk}(s,a)-\nu A^{\polk}_{C}(s,a))-\log\polk(a|s))] \\
        &= \min_{\lambda,\nu\geq 0}\;\lambda\delta+\nu \Tilde{b} + \lambda\E_{\substack{s\sim\dpolk\\ a\sim\pi^*}}[\log Z_{\lambda,\nu}(s)]
    \end{aligned}
\end{equation*}
\end{proof}

\section{Proof of Corollary \ref{corollary:loss_grad}}\label{append:loss_grad}
\lossgrad*
\begin{proof}
We only need to calculate the gradient of the loss function for a single sampled $s$. We first note that,
\begin{align*}
    \KL{\pol}{\pi^*}[s] =& -\sum_a\pol(a|s)\log\pi^*(a|s) + \sum_a\pol(a|s)\log \pol(a|s) \\
    =& H(\pol,\pi^*)[s] - H(\pol)[s]
\end{align*}
where $H(\pol)[s]$ is the entropy and $H(\pol,\pi^*)[s]$ is the cross-entropy under state $s$. We expand the cross entropy term which gives us
\begin{align*}
    H(\pol,\pi^*)[s] =& -\sum_a\pol(a|s)\log\pi^*(a|s) \\
    =& -\sum_a\pol(a|s)\log\left(\frac{\polk(a|s)}{Z_{\lambda,\nu}(s)}\exp\left[\frac{1}{\lambda}\left(A^{\polk}(s,a)-\nu A^{\polk}_{C}(s,a) \right)\right]\right) \\
    =& -\sum_a\pol(a|s)\log\polk(a|s) + \log Z_{\lambda,\nu}(s) - \frac{1}{\lambda}\sum_a\pol(a|s)\left(A^{\polk}(s,a)-\nu A^{\polk}_{C}(s,a) \right)
\end{align*}
We then subtract the entropy term to recover the KL divergence:
\begin{align*}
    \KL{\pol}{\pi^*}[s] =& \KL{\pol}{\polk}[s] + \log Z_{\lambda,\nu}(s) - \frac{1}{\lambda}\sum_a\pol(a|s)\left(A^{\polk}(s,a)-\nu A^{\polk}_{C}(s,a) \right) \\
    =& \KL{\pol}{\polk}[s] + \log Z_{\lambda,\nu}(s) - \frac{1}{\lambda}\E_{a\sim\polk(\cdot|s)}\left[\frac{\pol(a|s)}{\polk(a|s)}\left(A^{\polk}(s,a)-\nu A^{\polk}_{C}(s,a) \right)\right]
\end{align*}
where in the last equality we applied importance sampling to rewrite the expectation w.r.t. $\polk$. Finally, taking the gradient on both sides gives us:
\begin{equation*}
    \grad\KL{\pol}{\pi^*}[s] = \grad\KL{\pol}{\polk}[s] - \frac{1}{\lambda}\E_{a\sim\polk(\cdot|s)}\left[\frac{\grad\pol(a|s)}{\polk(a|s)}\left(A^{\polk}(s,a)- \nu A^{\polk}_{C}(s,a) \right)\right].
\end{equation*}
\end{proof}

\section{Proof of Corollary \ref{corollary:lagrange_grad}}\label{append:lagrange_grad}
\lagrangegrad*
\begin{proof}
From Theorem 1, we have
\begin{equation}
    L(\pi^*,\lambda,\nu) = \lambda\delta+\nu\Tilde{b} + \lambda\E_{\substack{s\sim\dpolk\\ a\sim\pi^*}}[\log Z_{\lambda,\nu}(s)].
\end{equation}
The first two terms is an affine function w.r.t. $\nu$, therefore its derivative is $\Tilde{b}$. We will then focus on the expectation in the last term. To simplify our derivation, we will first calculate the derivative of $\pi^*$ w.r.t. $\nu$,
\begin{equation*}
    \begin{aligned}
        \frac{\partial\pi^*(a|s)}{\partial\nu} &= \frac{\polk(a|s)}{Z^2_{\lambda,\nu}(s)}\bigg[Z_{\lambda,\nu}(s)\gradnu\exp\left(\frac{1}{\lambda}\left(A^{\polk}(s,a)-\nu A^{\polk}_{C}(s,a)\right)\right)\\
        &-\exp\left(\frac{1}{\lambda}\left(A^{\polk}(s,a)-\nu A^{\polk}_{C}(s,a)\right)\right)\frac{\partial Z_{\lambda,\nu}(s)}{\partial\nu}\bigg] \\
        &=-\frac{A^{\polk}_{C}(s,a)}{\lambda}\pi^*(a|s)-\pi^*(a|s)\frac{\partial\log Z_{\lambda,\nu}(s) }{\partial \nu}
    \end{aligned}
\end{equation*}
Therefore the derivative of the expectation in the last term of $L(\pi^*,\lambda,\nu)$ can be written as
\begin{equation}\label{eq:exp_grad}
    \begin{aligned}
        &\gradnu \E_{\substack{s\sim\dpolk\\ a\sim\pi^*}}[\log Z_{\lambda,\nu}(s)] \\
        =& \E_{\substack{s\sim\dpolk\\ a\sim\polk}}\left[\gradnu\left(\frac{\pi^*(a|s)}{\polk(a|s)}\log Z_{\lambda,\nu}(s)\right)\right] \\
        =& \E_{\substack{s\sim\dpolk\\ a\sim\polk}}\left[\frac{1}{\polk(a|s)}\left(\frac{\partial\pi^*(a|s)}{\partial\nu}\log Z_{\lambda,\nu}(s)+\pi^*(a|s)\frac{\partial\log Z_{\lambda,\nu}(s) }{\partial \nu}\right)\right] \\
        =& \E_{\substack{s\sim\dpolk\\ a\sim\polk}}\left[\frac{\pi^*(a|s)}{\polk(a|s)}\left(-\frac{A^{\polk}_{C}(s,a)}{\lambda}\log Z_{\lambda,\nu}(s)-\frac{\partial\log Z_{\lambda,\nu}(s)}{\partial\nu}\log Z_{\lambda,\nu}(s)+\frac{\partial\log Z_{\lambda,\nu}(s)}{\partial\nu}\right)\right] \\
        =& \E_{\substack{s\sim\dpolk\\ a\sim\pi^*}}\left[-\frac{A^{\polk}_{C}(s,a)}{\lambda}\log Z_{\lambda,\nu}(s)-\frac{\partial\log Z_{\lambda,\nu}(s)}{\partial\nu}\log Z_{\lambda,\nu}(s)+\frac{\partial\log Z_{\lambda,\nu}(s)}{\partial\nu}\right].
    \end{aligned}
\end{equation}
Also,
\begin{equation}
    \begin{aligned}
        \frac{\partial Z_{\lambda,\nu}(s)}{\partial\nu} &= \gradnu\sum_a\polk(a|s)\exp\left(\frac{1}{\lambda}\left(A^{\polk}(s,a)-\nu A^{\polk}_{C}(s,a)\right)\right) \\
        &=\sum_a -\polk(a|s)\frac{A^{\polk}_{C}(s,a)}{\lambda}\exp\left(\frac{1}{\lambda}\left(A^{\polk}(s,a)-\nu A^{\polk}_{C}(s,a)\right)\right) \\
        &=\sum_a -\frac{A^{\polk}_{C}(s,a)}{\lambda}\frac{\polk(a|s)}{Z_{\lambda,\nu}(s)}\exp\left(\frac{1}{\lambda}\left(A^{\polk}(s,a)-\nu A^{\polk}_{C}(s,a)\right)\right)Z_{\lambda,\nu}(s)\\
        &= -\frac{Z_{\lambda,\nu}(s)}{\lambda}\E_{a\sim\pi^*(\cdot|s)}\left[A^{\polk}_{C}(s,a)\right].
    \end{aligned}
\end{equation}
Therefore,
\begin{equation}\label{eq:logz_grad}
    \frac{\partial\log Z_{\lambda,\nu}(s)}{\partial\nu} = \frac{\partial Z_{\lambda,\nu}(s)}{\partial\nu}\frac{1}{Z_{\lambda,\nu}(s)} = -\frac{1}{\lambda}\E_{a\sim\pi^*(\cdot|s)}\left[A^{\polk}_{C}(s,a)\right].
\end{equation}
Plugging \eqref{eq:logz_grad} into the last equality in \eqref{eq:exp_grad} gives us
\begin{equation}\label{eq:exp_grad_final}
    \begin{aligned}
        \gradnu \E_{\substack{s\sim\dpolk\\ a\sim\pi^*}}[\log Z_{\lambda,\nu}(s)] &= \E_{\substack{s\sim\dpolk\\ a\sim\pi^*}}\left[-\frac{A^{\polk}_{C}(s,a)}{\lambda}\log Z_{\lambda,\nu}(s)+\frac{A^{\polk}_{C}(s,a)}{\lambda}\log Z_{\lambda,\nu}(s)-\frac{1}{\lambda}A^{\polk}_{C}(s,a)\right] \\
        &= -\frac{1}{\lambda}\E_{\substack{s\sim\dpolk\\ a\sim\pi^*}}\left[A^{\polk}_{C}(s,a)\right].
    \end{aligned}
\end{equation}
Combining \eqref{eq:exp_grad_final} with the derivatives of the affine term gives us the final desired result.

\end{proof}

\section{PPO Lagrangian and TRPO Lagrangian}
\subsection{PPO-Lagrangian}\label{append:ppo-l}
Recall that the PPO (clipped) objective takes the form \citep{schulman2017proximal}
\begin{equation}
    L(\theta) = \min\left(\frac{\pol(a|s)}{\polk(a|s)}\advk(s,a), \clip\left(\frac{\pol(a|s)}{\polk(a|s)}, 1-\epsilon, 1+\epsilon\right)\advk(s,a) \right)
\end{equation}
We augment this objective with an additional term to form the a new objective function
\begin{equation}\label{eq:PPO_L}
    \Tilde{L}(\theta) = L(\theta) +  \nu \left(J_C(\pol) -b\right)
\end{equation}
The Lagrangian method involves a maximization and a minimization step. For the maximization step, we optimize the objective \eqref{eq:PPO_L} by performing backpropagation w.r.t. $\theta$. For the minimization step, we apply gradient descent to the same objective w.r.t $\nu$. Like FOCOPS, the PPO-Lagrangian algorithm is also first-order thus simple to implement, however its empirical performance deteriorates drastically on more challenging environments (See Section \ref{sec:experiments}). Furthermore, it remains an open question whether PPO-Lagrangian satisfies any worst-case constraint guarantees.

\subsection{TRPO-Lagrangian}\label{append:trpo-l}
Similar to the PPO-Lagrangian method, we instead optimize an augmented TRPO problem
\begin{align}
\underset{\theta}{\text{maximize}} \quad
&  \Tilde{L}(\theta) \label{eq:TRPO_L_obj}  \\
\text{subject to} \quad
& \avKL{\pol}{\polk}[s]\leq \delta. \label{eq:TRPO_L_const}
\end{align}
where
\begin{equation}
    \Tilde{L}(\theta) = \frac{\pol(a|s)}{\polk(a|s)} A^{\polk}(s,a) + \nu \left(J_C(\pol)-b\right)
\end{equation}
We then apply Taylor approximation to \eqref{eq:TRPO_L_obj} and \eqref{eq:TRPO_L_const} which gives us
\begin{align}
\underset{\theta}{\text{maximize}} \quad
&  \Tilde{g}^T (\theta-\theta_k) \label{eq:TRPO_L_taylor_obj}  \\
\text{subject to} \quad
& \frac{1}{2}(\theta-\theta_k)^T H (\theta-\theta_k)\leq\delta. \label{eq:TRPO_L_taylor_const}
\end{align}
Here $\Tilde{g}$ is the gradient of \eqref{eq:TRPO_L_obj} w.r.t. the parameter $\theta$.

Like for the PPO-Lagrangian method, we first perform a maximization step where we optimize \eqref{eq:TRPO_L_taylor_obj}-\eqref{eq:TRPO_L_taylor_const} using the TRPO method \citep{schulman2015trust}. We then perform a minimization step by updating $\nu$. TRPO-Lagrangian is also a second-order algorithm. Performance-wise TRPO-Lagrangian does poorly in terms of constraint satisfaction. Like PPO-Lagrangian, further research into the theoretical properties of TRPO merits further research.

\section{FOCOPS for Different Cost Thresholds}

In this section, we verify that FOCOPS works effectively for different threshold levels. We experiment on the robots with speed limits environments. For each environment, we calculated the cost required for an unconstrained PPO agent after training for 1 million samples. We then used $25\%$, $50\%$, and $75\%$ of this cost as our cost thresholds and trained FOCOPS on each of thresholds respectively. The learning curves are reported in Figure \ref{fig:focops_thresholds}. We note from these plots that FOCOPS can effectively learn constraint-satisfying policies for different cost thresholds.

\begin{figure}[H]
    \centering
    \includegraphics[width=\textwidth]{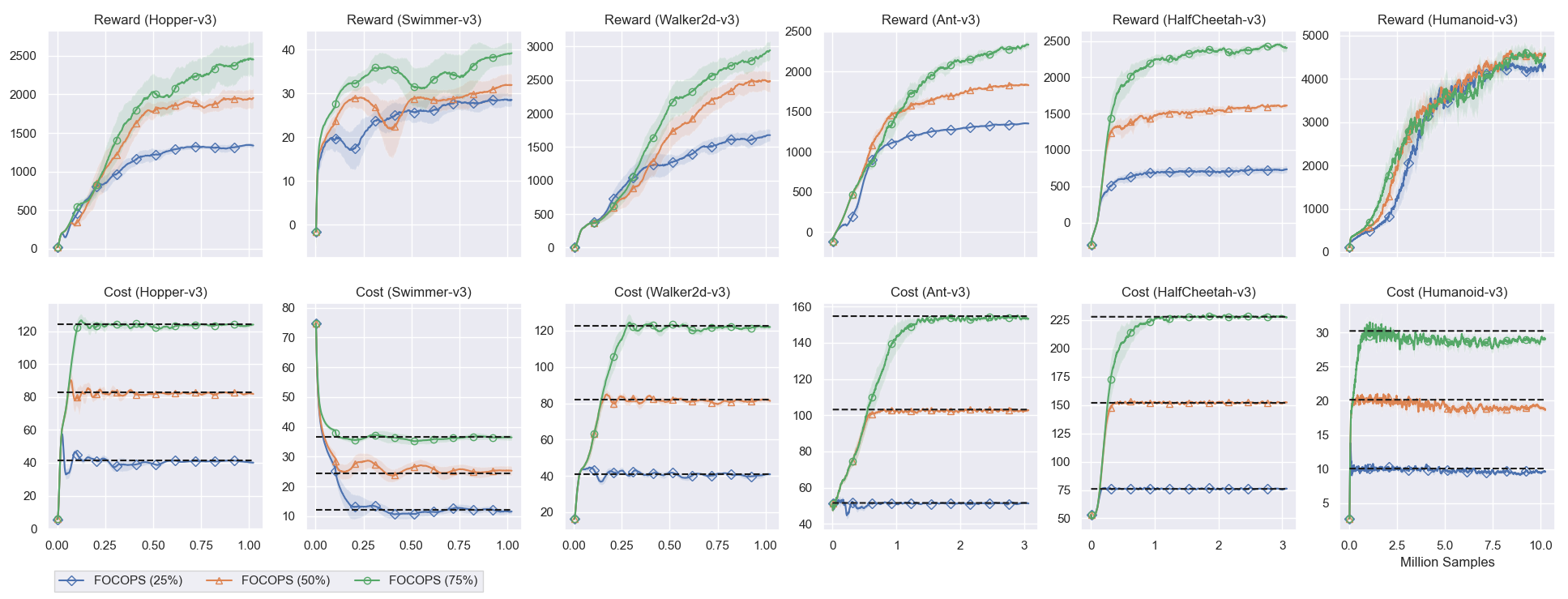}
    \caption{Performance of FOCOPS on robots with speed limit tasks with different cost thresholds. The $x$-axis represent the number of samples used and the $y$-axis represent the average total reward/cost return of the last 100 episodes. The solid line represent the mean of 1000 bootstrap samples over 10 random seeds. The horizontal lines in the cost plots represent the cost thresholds corresponding to $25\%$, $50\%$, and $75\%$ of the cost required by an unconstrained PPO agent trained with 1 million samples. Each solid line represents FOCOPS trained with the corresponding thresholds. The shaded regions represent the bootstrap normal $95\%$ confidence interval. Each of the solid lines represent}
    \label{fig:focops_thresholds}
\end{figure}

\section{Pseudocode}\label{append:pseudocode}

\begin{algorithm}[H]
\algtext*{End}
\begin{algorithmic}
\caption{First Order Constrained Optimization in Policy Space (FOCOPS)}
\INPUT{Policy network $\pol$; Value network for return $V_{\phi}$; Value network for costs $V_{\psi}^{C}$.}
\INPUT{Discount rates $\gamma$, GAE parameter $\beta$; Learning rates $\alpha_{\nu},\alpha_V,\alpha_{\pi}$;  Temperature $\lambda$; Initial cost constraint parameter $\nu$; Cost constraint parameter bound $\nu_{\max}$. Trust region bound $\delta$; Cost bound $b$.}
\While{Stopping criteria not met}
    \State Generate batch data of $M$ episodes of length $T$ from $(s_{i,t}, a_{i,t}, r_{i,t}, s_{i,t+1}, c_{i,t})$ from $\pol$, 
    \State $i=1,\dots,M$, $t=1,\dots,T$.
    \State Estimate $C$-return by averaging over $C$-return for all episodes:
    \[
    \hat{J}_{C} = \frac{1}{M}\sum_{i=1}^M\sum_{t=0}^{T-1} \gamma^t  c_{i,t}
    \]
	\State Store old policy $\theta'\gets\theta$
	\State Estimate advantage functions $\hat{A}_{i,t}$ and $\hat{A}_{i,t}^{C}$, $i=1,\dots,M$, $t=1,\dots,T$ using GAE.
	\State Get $V^{\text{target}}_{i,t}=\hat{A}_{i,t}+V_{\phi}(s_{i,t})$ and $V^{C, \text{target}}_{i,t}=\hat{A}_{i,t}+V_{\psi}^C(s_{i,t})$
	\State Update $\nu$ by 
	\[
	\nu\gets\proj_{\nu}\left[\nu - \alpha_{\nu}\left(b-\hat{J}_C\right)\right]
	\]
	\For{$K$ epochs}
	    \For{each minibatch $\{s_j, a_j, A_j, A^C_j,V^{\text{target}}_j,V^{C, \text{target}}_j\}$ of size $B$}
	    \State Value loss functions
	    \begin{align*}
	        \cL_V(\phi) &= \frac{1}{2N}\sum_{j=1}^B (V_{\phi}(s_j)-V^{\text{target}}_j)^2 \\
	        \cL_{V^C}(\psi) &= \frac{1}{2N}\sum_{j=1}^B (V_{\psi}(s_j)-V^{C,\text{target}}_j)^2
	    \end{align*}
	    \State Update value networks
	    \begin{align*}
	        \phi &\gets\phi -\alpha_V\nabla_{\phi}\cL_V(\phi) \\
	        \psi &\gets\psi -\alpha_V\nabla_{\psi}\cL_{V^C}(\psi) \\
	    \end{align*}
	    \State Update policy
	    \[
	    \theta \gets\theta -\alpha_{\pi}\hat{\nabla}_{\theta} \cL_{\pi}(\theta)
	    \]
	    \State where
	    \[
	    \hat{\nabla}_{\theta} \cL_{\pi}(\theta) 
    \approx \frac{1}{B}\sum_{j=1}^B\bigg[\grad\KL{\pol}{\pi_{\theta'}}[s_j] -\frac{1}{\lambda}\frac{\grad\pol(a_j|s_j)}{\pi_{\theta'}(a_j|s_j)}\bigg(\hat{A}_j - \nu\hat{A}^{C}_j \bigg)\bigg]\1_{\KL{\pol}{\pi_{\theta'}}[s_j]\leq\delta}
	    \]
	    \EndFor
	    \If{$\frac{1}{MT}\sum_{i=1}^M\sum_{t=0}^{T-1}\KL{\pol}{\pi_{\theta'}}[s_{i,t}]>\delta$}
	        \State Break out of inner loop
	    \EndIf
	\EndFor
\EndWhile
\end{algorithmic}
\end{algorithm}

\section{Implementation Details for Experiments}\label{append:experiments}

Our open-source implementation of FOCOPS can be found at  \url{https://github.com/ymzhang01/focops}. All experiments were implemented in Pytorch 1.3.1 and Python 3.7.4 on Intel Xeon Gold 6230 processors. We used our own Pytorch implementation of CPO based on \url{https://github.com/jachiam/cpo}. For PPO, PPO Lagrangian, TRPO Lagrangian, we used an optimized PPO and TRPO implementation based on \url{https://github.com/Khrylx/PyTorch-RL}, \url{https://github.com/ikostrikov/pytorch-a2c-ppo-acktr-gail}, and \url{https://github.com/ikostrikov/pytorch-trpo}.

\subsection{Robots with Speed Limit}\label{append:robots}

\subsubsection{Environment Details}
We used the MuJoCo environments provided by OpenAI Gym \cite{brockman2016openai} for this set of experiments. For agents manuvering on a two-dimensional plane, the cost is calculated as
\[
C(s,a) = \sqrt{v_x^2 + v_y^2}
\]
For agents moving along a straight line, the cost is calculated as
\[
C(s, a) = |v_x|
\]
where $v_x,v_y$ are the velocities of the agent in the $x$ and $y$ directions respectively.

\subsubsection{Algorithmic Hyperparameters}

We used a two-layer feedforward neural network with a $\tanh$ activation for both our policy and value networks. We assume the policy is Gaussian with independent action dimensions. The policy networks outputs a mean vector and a vector containing the state-independent log standard deviations. States are normalized by the running mean the running standard deviation before being fed to any network. The advantage values are normalized by the batch mean and batch standard deviation before being used for policy updates. Except for the learning rate for $\nu$ which is kept fixed, all other learning rates are linearly annealed to 0 over the course of training. Our hyperparameter choices are based on the default choices in the implementations cited at the beginning of the section. For FOCOPS, PPO Lagrangian, and TRPO Lagrangian, we tuned the value of $\nu_{\max}$ across $\{1,2,3,5,10,+\infty\}$ and used the best value for each algorithm. However we found all three algorithms are not especially sensitive to the choice of $\nu_{\max}$. Table \ref{tab:hyperparameters} summarizes the hyperparameters used in our experiments.
\begin{table}[h]
    \centering
    \caption{Hyperparameters for robots with speed limit experiments}
    \vskip 0.15in
    \begin{tabular}{l l l l l l}
    \toprule
       Hyperparameter & PPO & PPO-L & TRPO-L & CPO & FOCOPS \\
        \midrule
    No. of hidden layers & 2 & 2 & 2 & 2 & 2 \\
      No. of hidden nodes  & 64 & 64 & 64 & 64 & 64 \\
      Activation & $\tanh$ & $\tanh$ & $\tanh$ & $\tanh$ & $\tanh$\\
      Initial log std & -0.5 & -0.5 & -1 & -0.5 & -0.5\\
      Discount for reward $\gamma$ & 0.99 & 0.99 & 0.99 & 0.99 & 0.99 \\
      Discount for cost $\gamma_{C}$ & 0.99 & 0.99 & 0.99 & 0.99 & 0.99 \\
      Batch size & 2048 & 2048 & 2048 & 2048 & 2048 \\
      Minibatch size & 64 & 64 & N/A & N/A & 64 \\
      No. of optimization epochs & 10 & 10 & N/A & N/A & 10 \\
      Maximum episode length & 1000 & 1000 & 1000 & 1000 & 1000 \\
      GAE parameter (reward) & 0.95 & 0.95 & 0.95 & 0.95 & 0.95 \\ 
      GAE parameter (cost) & N/A & 0.95 & 0.95 & 0.95 & 0.95 \\
      Learning rate for policy & $3\times 10^{-4}$ & $3\times 10^{-4}$ & N/A & N/A & $3\times 10^{-4}$\\
      Learning rate for reward value net & $3\times 10^{-4}$ & $3\times 10^{-4}$ & $3\times 10^{-4}$ & $3\times 10^{-4}$ & $3\times 10^{-4}$ \\
      Learning rate for cost value net & N/A & $3\times 10^{-4}$ & $3\times 10^{-4}$ & $3\times 10^{-4}$ & $3\times 10^{-4}$ \\
      Learning rate for $\nu$ & N/A & 0.01 & 0.01 & N/A & 0.01 \\
      $L2$-regularization coeff. for value net & $3\times 10^{-3}$ & $3\times 10^{-3}$ & $3\times 10^{-3}$ & $3\times 10^{-3}$ & $3\times 10^{-3}$ \\
      Clipping coefficient & 0.2 & 0.2 & N/A & N/A & N/A \\
      Damping coeff. & N/A & N/A & 0.01 & 0.01 & N/A \\
       Backtracking coeff. & N/A & N/A & 0.8 & 0.8 & N/A \\
       Max backtracking iterations & N/A & N/A & 10 & 10 & N/A \\
       Max conjugate gradient iterations & N/A & N/A & 10 & 10 & N/A\\
       Iterations for training value net\footnotemark & 1 & 1 & 80 & 80 & 1 \\
     Temperature $\lambda$ & N/A & N/A & N/A & N/A & 1.5 \\
      Trust region bound $\delta$ & N/A & N/A & 0.01 & 0.01 & 0.02 \\
      Initial $\nu$, $\nu_{\max}$ & N/A & 0, 1 & 0, 2 & N/A & 0, 2 \\
       \bottomrule
    \end{tabular}
    \label{tab:hyperparameters}
\end{table}

\footnotetext{for PPO, PPO-L, and FOCOPS, this refers to the number of iteration for training the value net per minibatch update.}

\subsection{Circle}\label{append:circle_experiments}
\subsubsection{Environment Details}\label{append:circle_experiments_env}
In the circle tasks, the goal is for an agent to move along the circumference of a circle while remaining within a safety region smaller than the radius of the circle. The exact geometry of the task is shown in Figure \ref{fig:circle_task_geometry}.
\begin{figure}[h]
    \centering
    \includegraphics[scale=0.5]{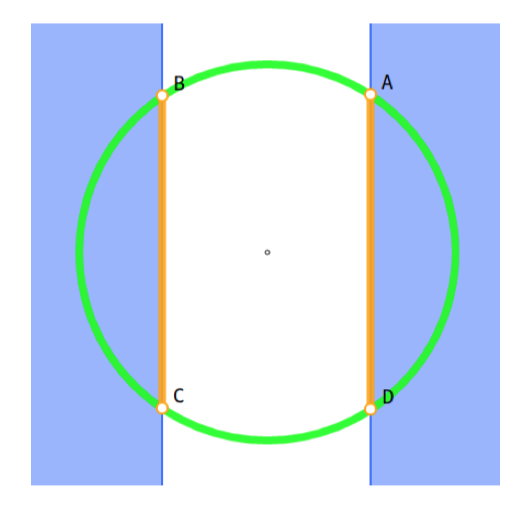}
    \caption{In the Circle task, reward is maximized by moving along the green circle. The agent is not allowed to enter the blue regions, so its optimal constrained path follows the line segments $AD$ and $BC$ (figure and caption taken from \cite{achiam2017constrained}).}
    \label{fig:circle_task_geometry}
\end{figure}
The reward and cost functions are defined as:
\begin{align*}
    R(s) &= \frac{-yv_x+xv_y}{1 + |\sqrt{x^2+y^2}-r|} \\
    C(s) &= \1(|x|>x_{\lim}).
\end{align*}
where $x,y$ are the positions of the agent on the plane, $v_x,v_y$ are the velocities of the agent along the $x$ and $y$ directions, $r$ is the radius of the circle, and $x_{\lim}$ specifies the range of the safety region. The radius is set to $r=10$ for both Ant and Humanoid while $x_{\lim}$ is set to 3 and 2.5 for Ant and Humanoid respectively. Note that these settings are identical to those of the circle task in \cite{achiam2017constrained}. Our experiments were implemented in OpenAI Gym \citep{brockman2016openai} while the circle tasks in \cite{achiam2017constrained} were implemented in rllab \citep{duan2016benchmarking}. We also excluded the Point agent from the original experiments since it is not a valid agent in OpenAI Gym. The first two dimensions in the state space are the $(x,y)$ coordinates of the center mass of the agent, hence the state space for both agents has two extra dimensions compared to the standard Ant and Humanoid environments from OpenAI Gym. Our open-source implementation of the circle environments can be found at \url{https://github.com/ymzhang01/mujoco-circle}.

\subsubsection{Algorithmic Hyperparameters}

For these tasks, we used identical settings as the robots with speed limit tasks except we used a batch size of 50000 for all algorithms and a minibatch size of 1000 for PPO, PPO-Lagrangian, and FOCOPS. The discount rate for both reward and cost were set to 0.995. For FOCOPS, we set $\lambda=1.0$ and $\delta=0.04$.

\section{Generalization Analysis}\label{append:generalization}
We used trained agents using all four algorithms (PPO Lagrangian, TRPO Lagrangian, CPO, and FOCOPS) on robots with speed limit tasks shown in Figure \ref{fig:mujoco_plot}. For each algorithm, we picked the seed with the highest maximum return of the last 100 episodes which does not violate the cost constraint at the end of training. The reasoning here is that for a fair comparison, we wish to pick the best performing seed for each algorithm. We then ran 10 episodes using the trained agents on 10 unseen random seeds (identical seeds are used for all four algorithms) to test how well the algorithms generalize over unseen data. The final results of running the trained agents on the speed limit and circle tasks are reported in Tables \ref{tab:mujoco_velocity_generalization}. We note that on unseen seeds FOCOPS outperforms the other three algorithms on five out of six tasks.

\begin{table}[h]
    \centering
    \caption{Average return of 10 episodes for trained agents on the robots with speed limit tasks on 10 unseen random seeds.  Results shown are the bootstrap mean and normal $95\%$ confidence interval with 1000 bootstrap samples.}
    \vskip 0.15in
    \begin{adjustbox}{width=\textwidth}
    \begin{tabular}{*6c}
    \toprule
    Environment   & {} & PPO-L  & TRPO-L & CPO & FOCOPS \\
    \midrule
    Ant-v3 &Reward  & $920.4\pm 75.9$ & $1721.4\pm 191.2$ & $1335.57\pm 43.17$ & $\boldsymbol{1934.9\pm 99.5}$ \\
    (103.12) &Cost  & $68.25\pm 11.05$ & $99.20\pm 2.55$ & $80.72\pm 3.82$ & $105.21\pm 5.91$ \\
    \hline
    HalfCheetah-v3 &Reward  & $1698.0\pm 22.5$ & $1922.4\pm 12.9$ & $1805.5\pm 60.0$ & \boldsymbol{$2184.3\pm 32.6$} \\
    (151.99) &Cost  & $150.21\pm 4.47$ & $179.82\pm 1.73$ & $164.67\pm 9.43$ & $158.39\pm 6.56$ \\
    \hline
    Hopper-v3 &Reward  & $2084.9\pm 39.69$ & $2108.8\pm 24.8$ & $\boldsymbol{2749.9\pm 47.0}$ & $2446.2\pm 9.0$ \\
    (82.75) &Cost  & $83.43\pm 0.41$ & $82.17\pm 1.53$ & $52.34\pm 1.95$ & $81.26\pm 0.88$ \\
    \hline
    Humanoid-v3 &Reward  & $582.2\pm 28.9$ & $3819.3\pm 489.2$ & $1814.8\pm 221.0$ & $\boldsymbol{4867.3\pm 350.8}$ \\
    (20.14) &Cost  & $18.93\pm 0.93$ & $18.60\pm 1.27$ & $20.30\pm 1.81$ & $21.58\pm 0.74$ \\
    \hline
    Swimmer-v3 &Reward  & $37.90\pm 1.05$ & $33.48\pm 0.44$ & $33.45\pm 2.30$ & $\boldsymbol{39.37\pm 2.04}$ \\
    (24.52) &Cost  & $25.49\pm 0.57$ & $32.81\pm 2.61$ & $22.61\pm 0.33$ & $17.23\pm 1.64$ \\
    \hline
    Walker2d-v3 &Reward  & $1668.7\pm 337.1$ & $2638.9\pm 163.3$ & $2141.7\pm 331.9$ & $\boldsymbol{3148.6\pm 60.5}$ \\
    (81.89) &Cost  & $79.23\pm 1.24$ & $90.96\pm 0.97$ & $40.67\pm 6.86$ & $73.35\pm 2.67$ \\
    \bottomrule
    \end{tabular}
    \end{adjustbox}
    \label{tab:mujoco_velocity_generalization}
\end{table}

\section{Sensitivity Analysis}\label{append:sensitive_analysis}

We tested FOCOPS across ten different values of $\lambda$, and five difference values of $\nu_{\max}$ while keeping all other parameters fixed by running FOCOPS for 1 millon samples on each of the robots with speed limit experiment. For ease of comparison, we normalized the values by the return and cost of an unconstrained PPO agent trained for 1 million samples (i.e. if FOCOPS achieves a return of $x$ and an unconstrained PPO agent achieves a result of $y$, the normalized result reported is $x/y$) The results on the robots with speed limit tasks are reported in Tables \ref{tab:lambda_FOCOPS} and \ref{tab:nu_max_FOCOPS}. We note that the more challenging environments such as Humanoid are more sensitive to parameter choices but overall FOCOPS is largely insensitive to hyperparameter choices (especially the choice of $\nu_{\max}$). We also presented the performance of PPO-L and TRPO-L for different values of $\nu_{\max}$.
\begin{table}[h]
    \centering
    \caption{Performance of FOCOPS for Different $\lambda$}
    \vskip 0.15in
    \begin{adjustbox}{width=\textwidth}
        \begin{tabular}{*{15}{c}}
\toprule 
{} & \multicolumn{2}{c}{Ant-v3} & \multicolumn{2}{c}{HalfCheetah-v3} & \multicolumn{2}{c}{Hopper-v3} & \multicolumn{2}{c}{Humanoid-v3} & \multicolumn{2}{c}{Swimmer-v3} & \multicolumn{2}{c}{Walker2d-v3} & \multicolumn{2}{c}{All Environments} \\ 
\cmidrule(lr){2-3}\cmidrule(lr){4-5}\cmidrule(lr){6-7}\cmidrule(lr){8-9}\cmidrule(lr){10-11}\cmidrule(lr){12-13}\cmidrule(lr){14-15}
$\lambda$ & Reward & Cost & Reward  & Cost & Reward  & Cost & Reward  & Cost & Reward  & Cost & Reward  & Cost & Reward  & Cost \\ 
\midrule
0.1& 0.66& 0.55& 0.38& 0.46& 0.77& 0.50& 0.63& 0.52& 0.34& 0.51& 0.43& 0.48& 0.53& 0.50 \\
0.5& 0.77& 0.54& 0.38& 0.45& 0.97& 0.50& 0.71& 0.54& 0.36& 0.50& 0.66& 0.50& 0.64& 0.50 \\
1.0& 0.83& 0.55& 0.47& 0.47& 1.04& 0.50& 0.80& 0.52& 0.34& 0.49& 0.76& 0.49& 0.70& 0.50 \\
1.3& 0.83& 0.55& 0.42& 0.47& 1.00& 0.50& 0.85& 0.53& 0.36& 0.51& 0.87& 0.49& 0.72& 0.51 \\
1.5& 0.83& 0.55& 0.42& 0.47& 1.01& 0.50& 0.87& 0.52& 0.37& 0.51& 0.87& 0.50& 0.73& 0.51 \\
2.0& 0.83& 0.55& 0.42& 0.47& 1.06& 0.50& 0.89& 0.52& 0.37& 0.52& 0.82& 0.45& 0.73& 0.51 \\
2.5& 0.79& 0.54& 0.43& 0.47& 1.03& 0.50& 0.94& 0.53& 0.35& 0.50& 0.73& 0.49& 0.71& 0.51 \\
3.0& 0.76& 0.54& 0.42& 0.47& 1.01& 0.49& 0.92& 0.52& 0.41& 0.50& 0.77& 0.49& 0.72& 0.50 \\
4.0& 0.70& 0.54& 0.40& 0.46& 1.00& 0.49& 0.87& 0.53& 0.43& 0.49& 0.64& 0.49& 0.67& 0.50 \\
5.0& 0.64& 0.55& 0.40& 0.47& 1.01& 0.50& 0.81& 0.54& 0.38& 0.49& 0.57& 0.50& 0.63& 0.51 \\
\bottomrule
\end{tabular} 
\end{adjustbox}
    \label{tab:lambda_FOCOPS}
\end{table}

\begin{table}[h]
    \centering
    \caption{Performance of FOCOPS for Different $\nu_{\max}$}
    \vskip 0.15in
    \begin{adjustbox}{width=\textwidth}
        \begin{tabular}{*{15}{c}}
\toprule 
{} & \multicolumn{2}{c}{Ant-v3} & \multicolumn{2}{c}{HalfCheetah-v3} & \multicolumn{2}{c}{Hopper-v3} & \multicolumn{2}{c}{Humanoid-v3} & \multicolumn{2}{c}{Swimmer-v3} & \multicolumn{2}{c}{Walker2d-v3} & \multicolumn{2}{c}{All Environments} \\ 
\cmidrule(lr){2-3}\cmidrule(lr){4-5}\cmidrule(lr){6-7}\cmidrule(lr){8-9}\cmidrule(lr){10-11}\cmidrule(lr){12-13}\cmidrule(lr){14-15}
$\nu_{\max}$ & Reward & Cost & Reward  & Cost & Reward  & Cost & Reward  & Cost & Reward  & Cost & Reward  & Cost & Reward  & Cost \\ 
\midrule
1& 0.83& 0.55& 0.45& 0.61& 1.00& 0.51& 0.87& 0.52& 0.40& 0.62& 0.88& 0.50& 0.74& 0.55 \\
2& 0.83& 0.55& 0.42& 0.47& 1.01& 0.50& 0.87& 0.52& 0.35& 0.51& 0.87& 0.50& 0.73& 0.51 \\
3& 0.81& 0.54& 0.41& 0.47& 1.01& 0.49& 0.83& 0.53& 0.34& 0.49& 0.87& 0.50& 0.71& 0.50 \\
5& 0.82& 0.55& 0.41& 0.47& 1.01& 0.50& 0.83& 0.53& 0.31& 0.49& 0.87& 0.50& 0.71& 0.51 \\
10& 0.82& 0.55& 0.41& 0.47& 1.01& 0.50& 0.83& 0.53& 0.34& 0.47& 0.87& 0.50& 0.71& 0.50 \\
$+\infty$ & 0.82& 0.55& 0.41& 0.47& 1.01& 0.50& 0.83& 0.53& 0.35& 0.47& 0.88& 0.50& 0.72& 0.50 \\
\bottomrule
\end{tabular} 
    \end{adjustbox}
    \label{tab:nu_max_FOCOPS}
\end{table}

\begin{table}[h]
    \centering
    \caption{Performance of PPO Lagrangian for Different $\nu_{\max}$}
    \vskip 0.15in
    \begin{adjustbox}{width=\textwidth}
        \begin{tabular}{*{15}{c}}
\toprule 
{} & \multicolumn{2}{c}{Ant-v3} & \multicolumn{2}{c}{HalfCheetah-v3} & \multicolumn{2}{c}{Hopper-v3} & \multicolumn{2}{c}{Humanoid-v3} & \multicolumn{2}{c}{Swimmer-v3} & \multicolumn{2}{c}{Walker2d-v3} & \multicolumn{2}{c}{All Environments} \\ 
\cmidrule(lr){2-3}\cmidrule(lr){4-5}\cmidrule(lr){6-7}\cmidrule(lr){8-9}\cmidrule(lr){10-11}\cmidrule(lr){12-13}\cmidrule(lr){14-15}
$\nu_{\max}$ & Reward & Cost & Reward  & Cost & Reward  & Cost & Reward  & Cost & Reward  & Cost & Reward  & Cost & Reward  & Cost \\ 
\midrule
1& 0.80& 0.55& 0.41& 0.49& 0.98& 0.49& 0.73& 0.52& 0.28& 0.50& 0.77& 0.50& 0.66& 0.51 \\
2& 0.71& 0.49& 0.36& 0.50& 0.81& 0.48& 0.73& 0.52& 0.32& 0.50& 0.72& 0.50& 0.61& 0.50 \\
3& 0.78& 0.54& 0.36& 0.47& 0.73& 0.49& 0.73& 0.52& 0.40& 0.48& 0.72& 0.50& 0.62& 0.50 \\
5& 0.77& 0.53& 0.35& 0.47& 0.73& 0.49& 0.73& 0.52& 0.40& 0.49& 0.72& 0.50& 0.62& 0.50 \\
10& 0.77& 0.54& 0.36& 0.47& 0.73& 0.49& 0.73& 0.52& 0.40& 0.49& 0.72& 0.50& 0.62& 0.50 \\
$+\infty$ & 0.66& 0.54& 0.27& 0.45& 0.73& 0.49& 0.55& 0.47& 0.40& 0.49& 0.72& 0.50& 0.55& 0.49 \\
\bottomrule
\end{tabular} 
    \end{adjustbox}
    \label{tab:nu_max_ppol}
\end{table}

\begin{table}[h]
    \centering
    \caption{Performance of TRPO Lagrangian for Different $\nu_{\max}$}
    \vskip 0.15in
    \begin{adjustbox}{width=\textwidth}
    \begin{tabular}{*{15}{c}}
\toprule 
{} & \multicolumn{2}{c}{Ant-v3} & \multicolumn{2}{c}{HalfCheetah-v3} & \multicolumn{2}{c}{Hopper-v3} & \multicolumn{2}{c}{Humanoid-v3} & \multicolumn{2}{c}{Swimmer-v3} & \multicolumn{2}{c}{Walker2d-v3} & \multicolumn{2}{c}{All Environments} \\ 
\cmidrule(lr){2-3}\cmidrule(lr){4-5}\cmidrule(lr){6-7}\cmidrule(lr){8-9}\cmidrule(lr){10-11}\cmidrule(lr){12-13}\cmidrule(lr){14-15}
$\nu_{\max}$ & Reward & Cost & Reward  & Cost & Reward  & Cost & Reward  & Cost & Reward  & Cost & Reward  & Cost & Reward  & Cost \\ 
\midrule
$1$ & 0.71& 0.50& 0.70& 0.68& 0.61& 0.50& 0.68& 0.50& 0.43& 0.61& 0.48& 0.50& 0.61& 0.55 \\ 
$2$ & 0.70& 0.51& 0.50& 0.53& 0.39& 0.53& 0.68& 0.50& 0.33& 0.53& 0.36& 0.50& 0.49& 0.52 \\ 
$3$ & 0.70& 0.51& 0.52& 0.53& 0.41& 0.53& 0.68& 0.50& 0.30& 0.67& 0.35& 0.50& 0.49& 0.54 \\ 
$5$ & 0.70& 0.51& 0.49& 0.52& 0.36& 0.52& 0.68& 0.50& 0.23& 0.67& 0.35& 0.51& 0.47& 0.54 \\ 
$10$ & 0.70& 0.51& 0.48& 0.51& 0.34& 0.52& 0.68& 0.50& 0.31& 0.77& 0.34& 0.50& 0.47& 0.55 \\ 
$+\infty$ & 0.70& 0.51& 0.48& 0.51& 0.36& 0.52& 0.68& 0.50& 0.30& 0.78& 0.34& 0.50& 0.48& 0.55 \\ 
\bottomrule
\end{tabular} 
    \end{adjustbox}
    \label{tab:nu_max_trpol}
\end{table}

\end{appendices}

\end{document}